\documentclass[conference]{IEEEtran}
\usepackage{times}

\usepackage{graphicx}
\usepackage{caption}
\graphicspath{
    {figures/}
}

\usepackage{amsmath,amssymb,enumerate}

\usepackage{amsthm}
\usepackage{setspace}
\usepackage{booktabs}
\usepackage[usenames,dvipsnames,svgnames,table]{xcolor}
\usepackage{mathtools}
\usepackage{algorithm, algorithmicx, algpseudocode}
\usepackage{blindtext}
\usepackage{gensymb}
\usepackage{xparse}
\usepackage{lipsum}
\usepackage{mathrsfs}
\usepackage[mathscr]{euscript}
\usepackage{times}
\usepackage{cite} 
\usepackage[numbers]{natbib}
\usepackage{multicol}
\definecolor{darkgreen}{rgb}{0,0.6,0}
\usepackage[bookmarks=true,colorlinks=true,pdfpagemode=UseNone,citecolor=darkgreen,linkcolor=black,urlcolor=BrickRed,pagebackref]{hyperref}
\usepackage[caption=false,font=footnotesize]{subfig}
\usepackage{amsfonts}
\usepackage{fixltx2e}
\usepackage{cleveref}
\usepackage[utf8]{inputenc}
\usepackage[T1]{fontenc}
\usepackage{textcomp}

\newtheorem{problem}{Problem}
\newtheorem{theorem}{Theorem}
\newtheorem{corollary}[theorem]{Corollary}

\newtheorem{lemma}[theorem]{Lemma}
\newtheorem{remark}{Remark}

\newtheorem{definition}{Definition}

\renewcommand{\thefigure}{\arabic{figure}}

\definecolor{note}{rgb}{0.1,0.1,1}
\definecolor{rephase}{rgb}{0.15,0.7,0.15}
\definecolor{bag}{rgb}{0.6,0.6,0.2}

\renewcommand{\H}{\mathcal{H}}
\DeclareMathOperator{\diff}{Diff}
\DeclareMathOperator{\tr}{Tr}
\DeclareMathOperator{\Lie}{Lie}
\newcommand{\se}{\mathfrak{se}}
\newcommand{\so}{\mathfrak{so}}
\newcommand{\I}{\mathcal{I}}
\newcommand{\sumxz}{\sum_{\substack{x_i\in X\\ z_j\in Z}}}

\newcommand{\Fcal}{\mathcal{F}}
\newcommand{\Gcal}{\mathcal{G}}
\newcommand{\Hcal}{\mathcal{H}}
\newcommand{\Ical}{\mathcal{I}}

\newcommand{\Xcal}{\mathcal{X}}

\newcommand{\transpose}{\mathsf{T}}
\newcommand{\SO}{\mathrm{SO}}
\newcommand{\SE}{\mathrm{SE}}
\newcommand{\GL}{\mathrm{GL}}

\DeclareDocumentCommand{\vectorToSkew}{ O{} }{\left(#1\right)_\times}

\DeclareDocumentCommand{\vector}{ O{} }{\mathrm{vec}(#1)}
\DeclareDocumentCommand{\zeros}{ O{} }{\textbf{0}_{#1}}
\DeclareDocumentCommand{\X}{ O{} O{} }{\textbf{X}_{#1}^{#2}}
\DeclareDocumentCommand{\XE}{ O{} O{} }{\hat{\textbf{X}}_{#1}^{#2}}
\DeclareDocumentCommand{\L}{ O{} O{} }{\textbf{L}_{#1}^{#2}}
\DeclareDocumentCommand{\S}{ O{} O{} }{\textbf{S}_{#1}^{#2}}
\DeclareDocumentCommand{\P}{ O{} O{} }{\textbf{P}_{#1}^{#2}}
\DeclareDocumentCommand{\B}{ O{} O{} }{\textbf{B}_{#1}^{#2}}
\DeclareDocumentCommand{\Q}{ O{} O{} }{\textbf{Q}_{#1}^{#2}}
\DeclareDocumentCommand{\H}{ O{} O{} }{\textbf{H}_{#1}^{#2}}
\DeclareDocumentCommand{\J}{ O{} O{} }{\textbf{J}_{#1}^{#2}}
\DeclareDocumentCommand{\R}{ O{} O{} }{\textbf{R}_{#1}^{#2}}
\DeclareDocumentCommand{\RE}{ O{} O{} }{\hat{\textbf{R}}_{#1}^{#2}}
\DeclareDocumentCommand{\Q}{ O{} O{} }{\textbf{Q}_{#1}^{#2}}
\DeclareDocumentCommand{\T}{ O{} O{} }{\textbf{T}_{#1}^{#2}}
\DeclareDocumentCommand{\L}{ O{} O{} }{\textbf{L}_{#1}^{#2}}
\DeclareDocumentCommand{\K}{ O{} O{} }{\textbf{K}_{#1}^{#2}}
\DeclareDocumentCommand{\V}{ O{} O{} }{\textbf{V}_{#1}^{#2}}
\DeclareDocumentCommand{\N}{ O{} O{} }{\textbf{N}_{#1}^{#2}}
\DeclareDocumentCommand{\Y}{ O{} O{} }{\textbf{Y}_{#1}^{#2}}
\DeclareDocumentCommand{\F}{ O{} O{} }{\textbf{F}_{#1}^{#2}}
\DeclareDocumentCommand{\G}{ O{} O{} }{\textbf{G}_{#1}^{#2}}
\DeclareDocumentCommand{\A}{ O{} O{} }{\textbf{A}_{#1}^{#2}}
\DeclareDocumentCommand{\TH}{ O{} O{} }{\boldsymbol{\Theta}_{#1}^{#2}}
\DeclareDocumentCommand{\x}{ O{} O{} }{\textbf{x}_{#1}^{#2}}
\DeclareDocumentCommand{\e}{ O{} O{} }{\textbf{e}_{#1}^{#2}}
\DeclareDocumentCommand{\c}{ O{} O{} }{\textbf{c}_{#1}^{#2}}
\DeclareDocumentCommand{\C}{ O{} O{} }{\textbf{C}_{#1}^{#2}}
\DeclareDocumentCommand{\CM}{ O{} O{} }{\tilde{\textbf{C}}_{#1}^{#2}}
\DeclareDocumentCommand{\RM}{ O{} O{} }{\tilde{\textbf{R}}_{#1}^{#2}}
\DeclareDocumentCommand{\I}{ O{} O{} }{\textbf{I}_{#1}^{#2}}
\DeclareDocumentCommand{\O}{ O{} O{} }{\textbf{O}_{#1}^{#2}}
\DeclareDocumentCommand{\r}{ O{} O{} }{\textbf{r}_{#1}^{#2}}
\DeclareDocumentCommand{\t}{ O{} O{} }{\textbf{t}_{#1}^{#2}}
\DeclareDocumentCommand{\u}{ O{} O{} }{\textbf{u}_{#1}^{#2}}
\DeclareDocumentCommand{\d}{ O{} O{} }{\textbf{d}_{#1}^{#2}}
\DeclareDocumentCommand{\dE}{ O{} O{} }{\hat{\textbf{d}}_{#1}^{#2}}
\DeclareDocumentCommand{\b}{ O{} O{} }{\textbf{b}_{#1}^{#2}}
\DeclareDocumentCommand{\a}{ O{} O{} }{\textbf{a}_{#1}^{#2}}
\DeclareDocumentCommand{\g}{ O{} O{} }{\textbf{g}_{#1}^{#2}}
\DeclareDocumentCommand{\dM}{ O{} O{} }{\tilde{\textbf{d}}_{#1}^{#2}}
\DeclareDocumentCommand{\params}{ O{} O{} }{\boldsymbol{\theta}_{#1}^{#2}}
\DeclareDocumentCommand{\paramsE}{ O{} O{} }{\hat{\boldsymbol{\theta}}_{#1}^{#2}}
\DeclareDocumentCommand{\paramError}{ O{} O{} }{\boldsymbol{\zeta}_{#1}^{#2}}
\DeclareDocumentCommand{\gyroscopeBias}{ O{} }{\textbf{b}_{#1}^{g}}
\DeclareDocumentCommand{\gyroscopeBiasE}{ O{} }{\hat{\textbf{b}}_{#1}^{g}}
\DeclareDocumentCommand{\accelerometerBias}{ O{} }{\textbf{b}_{#1}^{a}}
\DeclareDocumentCommand{\accelerometerBiasE}{ O{} }{\hat{\textbf{b}}_{#1}^{a}}
\DeclareDocumentCommand{\position}{ O{} O{} }{{}_\text{#2}\textbf{p}_{\text{#1}}}
\DeclareDocumentCommand{\positionDot}{ O{} O{} }{{}_\text{#2}\dot{\textbf{p}}_{\text{#1}}}
\DeclareDocumentCommand{\linearVelocity}{ O{} O{} }{{}_\text{#2}\textbf{v}_{\text{#1}}}
\DeclareDocumentCommand{\linearVelocityM}{ O{} O{} }{{}_\text{#2}\tilde{\textbf{v}}_{\text{#1}}}
\DeclareDocumentCommand{\orientation}{ O{} }{\textbf{R}_{\text{#1}}}
\DeclareDocumentCommand{\orientationDot}{ O{} }{\dot{\textbf{R}}_{\text{#1}}}
\DeclareDocumentCommand{\angularVelocity}{ O{} O{} }{{}_\text{#2}\boldsymbol{\omega}_{\text{#1}}}
\DeclareDocumentCommand{\angularVelocityM}{ O{} O{} }{{}_\text{#2}\tilde{\boldsymbol{\omega}}_{\text{#1}}}
\DeclareDocumentCommand{\acceleration}{ O{} O{} }{{}_\text{#2}\textbf{a}_{\text{#1}}}
\DeclareDocumentCommand{\accelerationM}{ O{} O{} }{{}_\text{#2}\tilde{\textbf{a}}_{\text{#1}}}
\DeclareDocumentCommand{\p}{ O{} O{} }{\textbf{p}_{#1}^{#2}}
\DeclareDocumentCommand{\pE}{ O{} O{} }{\hat{\textbf{p}}_{#1}^{#2}}
\DeclareDocumentCommand{\pM}{ O{} O{} }{\tilde{\textbf{p}}_{#1}^{#2}}
\DeclareDocumentCommand{\v}{ O{} O{} }{\textbf{v}_{#1}^{#2}}
\DeclareDocumentCommand{\vE}{ O{} O{} }{\hat{\textbf{v}}_{#1}^{#2}}
\DeclareDocumentCommand{\w}{ O{} O{} }{\boldsymbol{\omega}_{#1}^{#2}}
\DeclareDocumentCommand{\wM}{ O{} O{} }{\tilde{\boldsymbol{\omega}}_{#1}^{#2}}
\DeclareDocumentCommand{\noise}{ O{} O{} }{\textbf{w}_{#1}^{#2}}
\DeclareDocumentCommand{\FK}{ O{} }{\;\textit{\textbf{h}}_{#1}}
\DeclareDocumentCommand{\FKswitch}{ O{} }{\;\textit{\textbf{g}}_{#1}}
\DeclareDocumentCommand{\angleTheta}{ O{} O{} }{\boldsymbol{\theta}_{#1}^{#2}}
\DeclareDocumentCommand{\anglePhi}{ O{} O{} }{\boldsymbol{\phi}_{#1}^{#2}}
\DeclareDocumentCommand{\encoders}{ O{} O{} }{\boldsymbol{\alpha}_{#1}^{#2}}
\DeclareDocumentCommand{\encodersM}{ O{} O{} }{\tilde{\boldsymbol{\alpha}}_{#1}^{#2}}
\DeclareDocumentCommand{\offsets}{ O{} O{} }{\boldsymbol{\epsilon}_{#1}^{#2}}
\DeclareDocumentCommand{\Cov}{ O{} O{} }{\boldsymbol{\Sigma}_{#1}^{#2}}
\DeclareDocumentCommand{\Axis}{ O{} O{} }{\textnormal{Axis}_{#1}^{#2}}
\DeclareDocumentCommand{\using}{ O{} O{} }{\stackrel{\mathmakebox[\widthof{=}]{\text{eq.} (#1)}}{#2} \enspace}
\DeclareDocumentCommand{\Adjoint}{ O{} }{\mathrm{Ad}_{#1}}
\DeclareDocumentCommand{\vectorToAlgebra}{ O{} }{\mathscr{L}_\mathfrak{g}\left(#1\right)}
\DeclareDocumentCommand{\groupError}{ O{} O{} }{\boldsymbol{\eta}_{#1}^{#2}}
\DeclareDocumentCommand{\tangentError}{ O{} O{} }{{\xi}_{#1}^{#2}}
\newcommand{\squeezeup}{\vspace{-2mm}}

\makeatletter
\newcommand{\mathleft}{\@fleqntrue\@mathmargin0pt}
\newcommand{\mathcenter}{\@fleqnfalse}
\makeatother

\begin{document}

\title{Continuous Direct Sparse Visual Odometry from RGB-D Images}

\author{Maani Ghaffari$^{*}$, William Clark$^{*}$, Anthony Bloch, Ryan M. Eustice, and Jessy W. Grizzle\\
University of Michigan, Ann Arbor, MI, USA\\
\tt\small \{maanigj, wiclark, abloch, eustice, grizzle\}@umich.edu\\
{\footnotesize $^{*}$The authors contributed equally to this work.}
}

\maketitle

\begin{abstract}
This paper reports on a novel formulation and evaluation of visual odometry from RGB-D images. Assuming a static scene, the developed theoretical framework generalizes the widely used direct energy formulation (photometric error minimization) technique for obtaining a rigid body transformation that aligns two overlapping RGB-D images to a continuous formulation. The continuity is achieved through functional treatment of the problem and representing the process models over \mbox{RGB-D} images in a reproducing kernel Hilbert space; consequently, the registration is not limited to the specific image resolution and the framework is fully analytical with a closed-form derivation of the gradient. We solve the problem by maximizing the inner product between two functions defined over RGB-D images, while the continuous action of the rigid body motion Lie group is captured through the integration of the flow in the corresponding Lie algebra. Energy-based approaches have been extremely successful and the developed framework in this paper shares many of their desired properties such as the parallel structure on both CPUs and GPUs, sparsity, semi-dense tracking, avoiding explicit data association which is computationally expensive, and possible extensions to the simultaneous localization and mapping frameworks. The evaluations on experimental data and comparison with the equivalent energy-based formulation of the problem confirm the effectiveness of the proposed technique, especially, when the lack of structure and texture in the environment is evident.
\end{abstract}

\IEEEpeerreviewmaketitle

\section{Introduction}
Sensor registration is a fundamental task in robotic perception. Modern algorithms for sensor registration rely on optimization techniques to solve often nonlinear and large-scale problems~\citep{kummerle2011g,dellaert2012factor}. In particular, estimating the rigid body transformation of a moving sensor such as a camera (or in combination with other sensors) is of major importance in robotics and computer vision~\citep{davison2007monoslam,strasdat2012local,engel2014lsd,forster2017manifold,rhartley-2018c}. Cameras in both \emph{monocular} and \emph{depth} versions are rich sources of information acquisition and are well-suited for varieties of real-world applications of autonomous systems that usually involves solving the Simultaneous Localization and Mapping (SLAM) problem~\citep{dellaert2006square,kaess2012isam2,engel2018direct}. 

\begin{figure}[t]
    \centering
    \subfloat{\includegraphics[width=0.5\columnwidth]{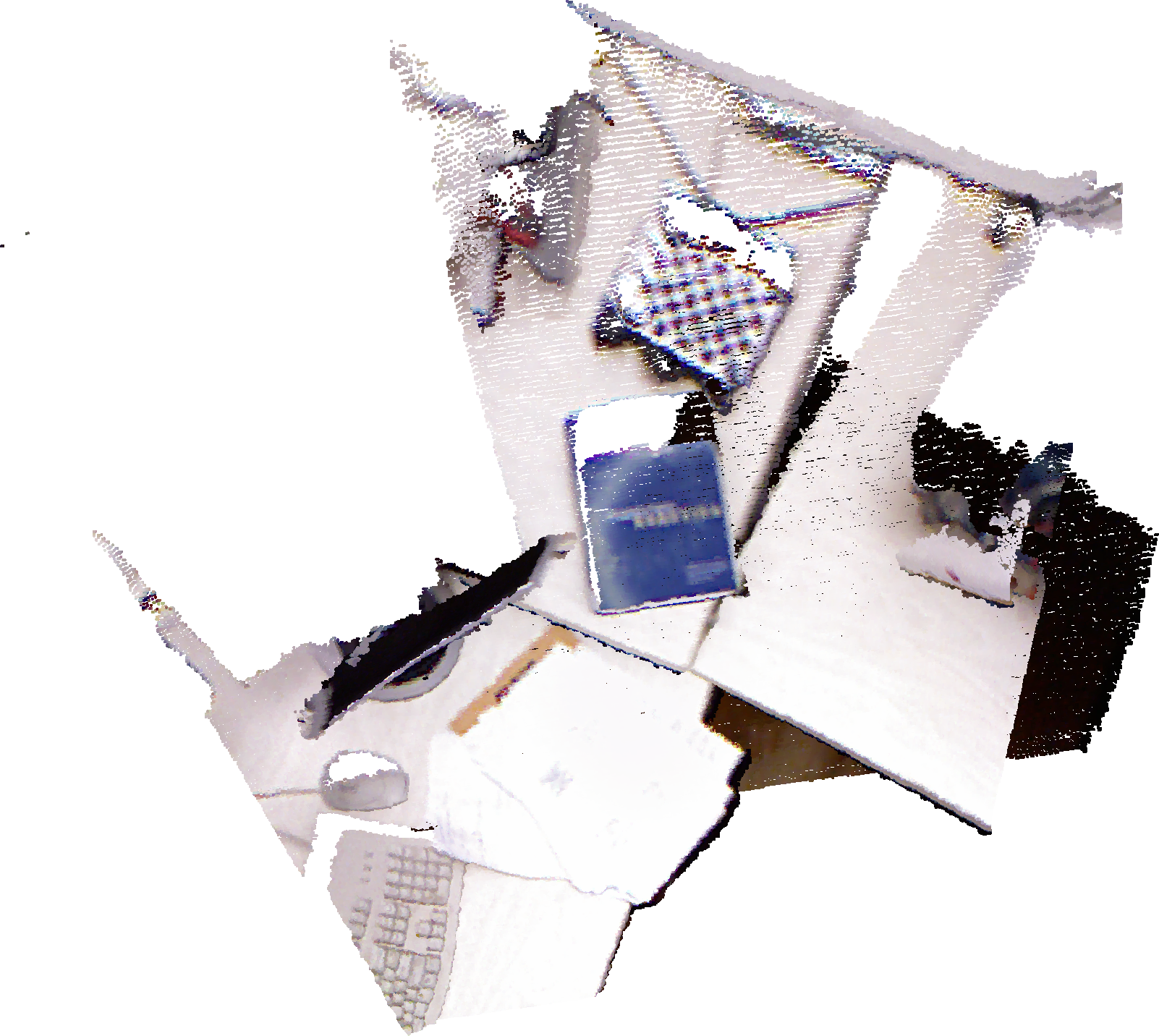}}
    \subfloat{\includegraphics[width=0.5\columnwidth]{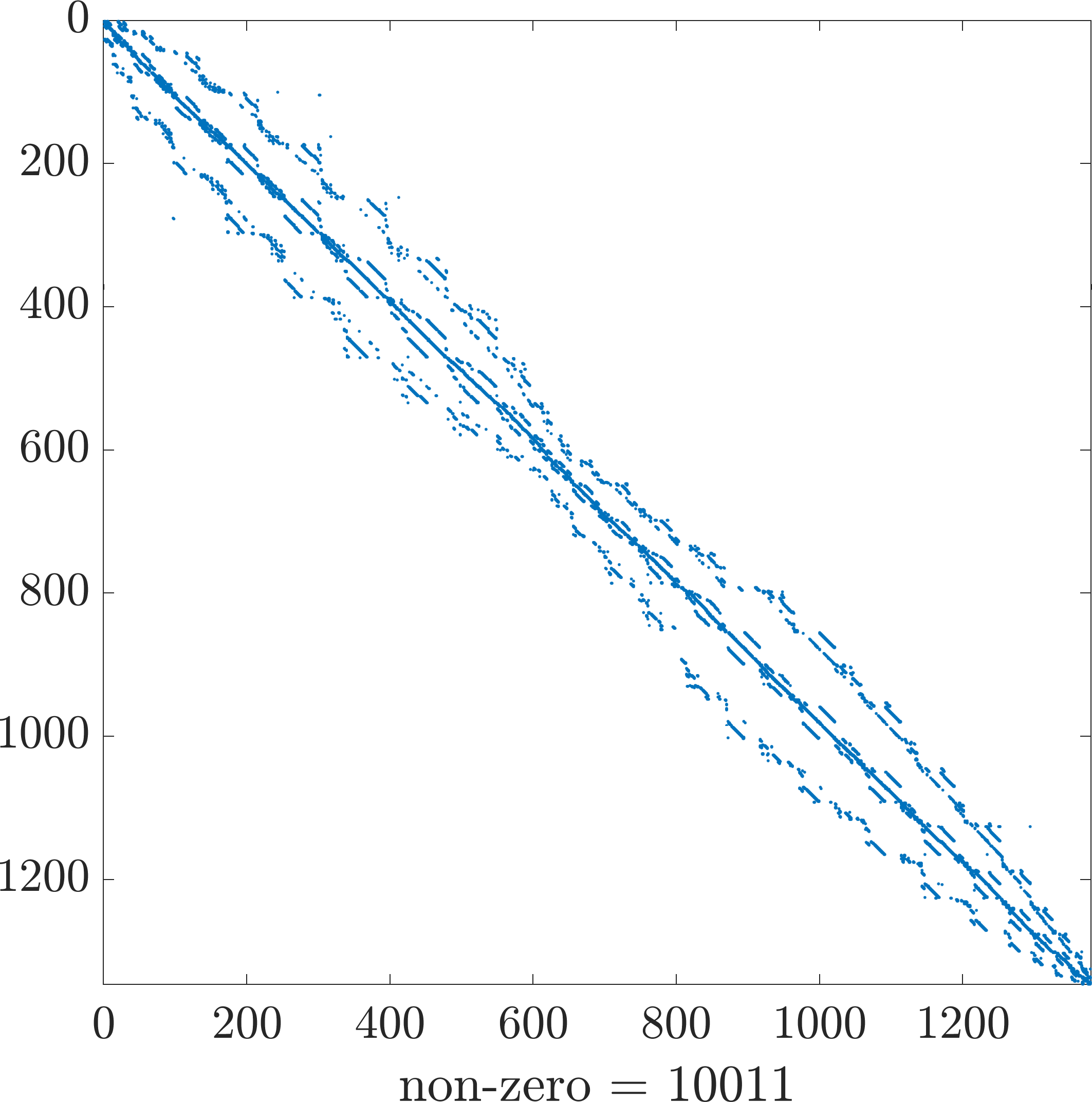}}
    \caption{The registration of frames 1 and 2 of fr1/desk sequence from \mbox{RGB-D} SLAM dataset and benchmark~\citep{sturm12iros} using the proposed continuous sensor registration framework. In the right, the figure shows the sparsity pattern of the corresponding kernel matrix of the inner product structure. The computation of this sparse matrix can be done entirely in parallel. Geometrically, the matrix shows the correlation of each measurement in frame 2 after applying the final rigid body transformation with other measurements in frame 1.}
    \label{fig:first}
\end{figure}

A key strength of the successful state-of-the-art visual perception algorithms is the ability to track the camera pose accurately with the lowest possible drift. Then, visual place-recognition algorithms such as the work of~\citet{galvez2012bags} are used to remove the drift by providing place-revisiting measurements known as \emph{loop-closures}. Direct visual odometry methods minimize the \emph{photometric error} using the photometric measurements, i.e., image intensity values, provided by the camera. As opposed to \emph{indirect} methods, the information contained in an image is not abstracted into a set of sparse \emph{keypoints} or \emph{features}, e.g., corners; hence, direct methods can achieve superior performance in tracking and \emph{dense} or \emph{semi-dense} mappings, given a well-calibrated camera~\citep{audras2011real,kerl2013robust,engel2014lsd}. In addition, direct visual odometry \emph{front-end systems} also possess a \emph{sparse Hessian structure} which is very similar to the general SLAM structure, leading to real-time performance. In contrast, indirect methods model the geometric error and are much more robust to geometric noise such as a poor intrinsic camera calibration or a rolling shutter~\citep{engel2018direct}. 


In this paper, inspired by the fact that (colored) point cloud measurements constitute discrete sensing information from the continuous motion of the camera and the idea of avoiding costly data association between measurements, we formulate the registration problem in a continuous form that directly operates on a mapping from 3D space to an abstract information space such as intensity surface. The explicit representation between color information and 2D/3D geometry (image or Euclidean space coordinates) is not directly available; hence, the current direct methods use numerical differentiation for computing the gradient and are limited to fixed image size and resolution, given the camera model and measurements for re-projection of the 3D points. As illustrated in Fig.~\ref{fig:first}, we show that the proposed continuous representation can be sparse or sparsified which means, geometrically speaking, the local correlation is sufficiently rich to capture the information available in data. This work has the following contributions:
\begin{enumerate}
    \item We develop a fundamentally novel formulation of the sensor registration problem that is continuous and models the action of an arbitrary Lie group on any smooth manifold. Our solution uses the integration of the flow in the Lie algebra by maximizing the inner product between two functions defined over the fixed and moving measurements sets. The continuity is achieved through functional treatment of the problem and representing the functions in a reproducing kernel Hilbert space.
    \item We apply the developed framework to the particular, and commonly used, case of RGB-D images (depth camera) and $\SE(n)$ matrix Lie group. As a result, the registration is not limited to the specific image resolution.
    \item As opposed to the current direct energy formulation, which involves computation of the numerical image intensity gradient to be used in conjunction with the analytical Jacobian of the pose via the chain rule, our framework is fully analytical and the gradient has a complete closed-form derivation.
    \item We evaluate the proposed algorithm using publicly available RGB-D benchmark in~\citet{sturm12iros} and provide the open-source implementation available at:\\ \href{https://github.com/MaaniGhaffari/cvo-rgbd}{\url{https://github.com/MaaniGhaffari/cvo-rgbd}}
\end{enumerate}

The remainder of this paper is organized as follows. In \textsection\ref{sec:prelim}, the required mathematical preliminaries and notation are briefly explained. The main theoretical result for the general problem formulation on any smooth manifold and any Lie group acting on it is given in \textsection\ref{sec:problem}. The specialized form of the problem formulation for the special Euclidean group is derived in \textsection\ref{sec:euclidean}. A brief theoretical analysis for the verification of the idea is provided in \textsection\ref{sec:verification}. The integration of the flow for $\SE(n)$ and for the special case of 3D space, $\SE(3)$, to obtain the solution is explained in \textsection\ref{sec:senintegration}. Experimental evaluations of the proposed method for registration and tracking using RGB-D images are presented in \textsection\ref{sec:results}. Finally, \textsection\ref{sec:conclusion} concludes the paper and provides suggestions as future work.


\section{Mathematical Preliminaries and Notation}
\label{sec:prelim}

We first review some mathematical preliminaries to establish the notation and give the paper better flow before stating the main results.

 

\subsection{Matrix Lie Group of Motion in $\mathbb{R}^n$}
 \label{subsec:liesen}
 
The general linear group of degree $n$, denoted \mbox{\small{$\GL_n(\mathbb{R})$}}, is the set of all $n\times n$ real invertible matrices, where the group binary operation is the ordinary matrix multiplication. The $n$-dimensional special orthogonal group, denoted  
\begin{equation}
 \begin{split}
  \nonumber \SO(n) = \{{R}\in \GL_n(\mathbb{R}) |~ {R} {R}^\transpose = {I}_n, \operatorname{det} {R} = +1\},
 \end{split}
\end{equation}
is the rotation group on $\mathbb{R}^n$. The $n$-dimensional special Euclidean group, denoted 
\begin{equation}
 \begin{split}
 \resizebox{\columnwidth}{!}{$
  \nonumber \SE(n) = \{ {h} = \left[\begin{array}{cc} {R} & {T} \\ 0 & 1 \end{array} \right] \in \GL_{n+1}(\mathbb{R}) |~ {R} \in \SO(n), {p} \in \mathbb{R}^n \},$}
 \end{split}
\end{equation}
is the group of rigid transformations, i.e., \emph{direct isometries}, on $\mathbb{R}^n$. A transformation such as \mbox{${h} \in \SE(3)$} is the parameter space  in many sensor registration problems which consists of a rotation and a translation components. Let \mbox{$\bar{{h}} \in \SE(3)$} be an estimate of ${h}$. We compute the rotational and translational distances using $\lVert{\log(\bar{{R}} {R}^{\transpose})}\rVert_{\mathrm{F}}$ and $\lVert{\bar{{T}} - \bar{{R}} {R}^{\transpose} {T}} \rVert$, respectively, where $\log(\cdot)$ is the Lie logarithm map which is, here, the matrix logarithm. These definitions are consistent with the transformation distance that can be directly computed using $\lVert{\log (\bar{{h}} {h}^{-1})}\rVert_{\mathrm{F}}$. Here, $\lVert A \rVert^2_{\mathrm{F}} = \tr(A^{\transpose}A)$ is the Frobenius norm.

This paper studies not only matrix Lie groups, but also their actions on manifolds. We have the following definition:
\begin{definition}
Let $\Gcal$ be a group and $X$ a set. A (left) group action of $\Gcal$ on $X$, denoted as $\Gcal\curvearrowright X$, is a group homomorphism $\Gcal\to \emph{Aut}(X)$ (automorphism of $X$). If $X$ is a smooth manifold, the action is smooth if $\Gcal\to \diff(X)$ (diffeomorphism of $X$).
\end{definition}
\begin{remark}
A group action can similarly be viewed as a function $\varphi:\Gcal\times X\to X$ satisfying two conditions (where $\varphi(g,x)$ will be denoted by $g.x$)
\begin{enumerate}
    \item Identity: If $e\in \Gcal$ is the identity element, then $e.x=x$ for all $x\in X$.
    \item Compatibility: $(gh).x = g.(h.x)$ for all $g,h\in \Gcal$ and $x\in X$.
\end{enumerate}
\end{remark}
\begin{remark}
This paper will be primarily focused on the standard action of $\SE(n)$ on $\mathbb{R}^n$ given by $(R,T).x = Rx+T$ for $R\in \SO(n)$ and $T\in \mathbb{R}^n$.
\end{remark}
\subsection{Hilbert Space}
Let $V$ be a vector space over the field of real numbers $\mathbb{R}$. An inner product on $V$ is a function \mbox{$\langle\cdot,\cdot\rangle: V\times V \to \mathbb{R}$} that is bilinear, symmetric, and positive definite. The pair $(V, \langle\cdot,\cdot\rangle)$ is called an inner product space. The inner product induces a norm that measures the magnitude or length of a vector:
	$\lVert {v} \rVert = \sqrt{\langle{v}, {v}\rangle}.$
The norm in turn induces a metric that allows for calculating the distance between two vectors:
	$d({v},{w}) = \lVert {v} - {w} \rVert = \sqrt{\langle{v} - {w}, {v} - {w}\rangle}.$
Such a metric is homogeneous; for $a\in\mathbb{R}$, $d(a {v}, a {w}) = |a| d({v}, {w})$, and translational invariant; $d({v} + {x}, {w} + {x}) = d({v}, {w})$. These properties are inherited from the induced norm. The distance metric is positive definite, symmetric, and satisfies the triangle inequality. In addition to measuring distances, it is important to be able to understand limits. This leads to the definition of a Cauchy sequence and completeness.

\begin{definition}[Cauchy Sequence]
A Cauchy sequence is a sequence $\{x_i\}_{i=1}^{\infty}$ such that for any real number $\epsilon > 0$ there exists a natural number $\bar{n} \in \mathbb{N}$ such that for some distance metric $d(x_n, x_m) < \epsilon$ for all $n,m > \bar{n}$.
\end{definition}

\begin{definition}[Completeness]
A metric space $(M,d)$ is complete if every Cauchy sequence in $M$ converges in $M$, i.e., to a limit that is in $M$.
\end{definition}

Such a metric space contains all its limit points. Note that completeness is with respect to the metric $d$ and not the topology of the space. Now, we can give a definition for a Hilbert space.

\begin{definition}[Hilbert Space]
A Hilbert space, $\Hcal$, is a complete inner product space; that is any Cauchy sequence, using the metric induced by the inner product, converges to an element in $\Hcal$.
\end{definition}

The reviewed definitions and properties are also valid for a vector space of functions. Let $(\Hcal, \langle\cdot,\cdot\rangle_{\Hcal})$ be a real Hilbert space of functions with the inner product between any two square-integrable functions $f, g \in \Hcal$ (or \mbox{$f, g \in L^2(\mathbb{R}, \mu)$}) defined as:
\begin{equation}
\label{eq:ipl2}
	\langle f, g\rangle_{\Hcal} := \int f({x}) g({x}) \mathrm{d}\mu({x}),
\end{equation}
where $\mu$ is the Lebesgue measure on $\mathbb{R}$. Similarly, the induced norm by the inner product is \mbox{$\lVert f \rVert_{\Hcal} = \sqrt{\langle f, f\rangle_{\Hcal}}$}. The Hilbert space of functions can be thought of as an infinite-dimensional counterpart of the finite-dimensional vector spaces discussed earlier.

\subsection{Representation and Reproducing Kernel Hilbert Space}
\label{sec:problem}

We now move to a more special type called Reproducing Kernel Hilbert Space (RKHS) which we will use in this work.

\begin{definition}[Kernel]
Let ${x}, {x}' \in \mathcal{X}$ be a pair of inputs for a function $k:\mathcal{X} \times \mathcal{X} \to \mathbb{R}$ known as the kernel. A kernel is symmetric if $k({x}, {x}')=k({x}',{x})$, and is positive definite if for any nonzero $f \in \Hcal~ (\text{or}~L^2(\Xcal,\mu))$: 
\begin{equation}
	\nonumber \int k({x}, {x}') f({x}) f({x}') \mathrm{d}\mu({x}) \mathrm{d}\mu({x}') > 0.
\end{equation}
\end{definition}

\begin{definition}[Reproducing Kernel Hilbert Space~\cite{berlinet2004reproducing}]
Let $\Hcal$ be a real-valued Hilbert space on a non-empty set $\Xcal$. A function $k:\Xcal \times \Xcal \to \mathbb{R}$ is a reproducing kernel of the Hilbert space $\Hcal$ iff:
\begin{enumerate}
	\item $\forall {x} \in \Xcal, \quad k(\cdot, {x}) \in \Hcal$,
	\item $\forall {x} \in \Xcal, \quad \forall f \in \Hcal \quad \langle f, k(\cdot, {x}) \rangle = f({x})$.
\end{enumerate}
The Hilbert space $\Hcal$ (RKHS) which possesses a reproducing kernel $k$ is called  a Reproducing Kernel Hilbert Space or a proper Hilbert space.
\end{definition}

The second property is called \emph{the reproducing property}; that is using the inner product of $f$ with $k(\cdot,{x})$, the value of function $f$ is reproduced at point ${x}$. Also, using both conditions we have: $\forall {x}, {z} \in \Xcal$, $k({x},{z}) = \langle k(\cdot, {x}), k(\cdot, {z}) \rangle$.

\begin{lemma}
Any reproducing kernel is a positive definite function~\cite{berlinet2004reproducing}.
\end{lemma}

Finding a reproducing kernel of an RKHS might seem difficult, but fortunately, there is a one-to-one relation between a reproducing kernel and its associated RKHS, and such a reproducing kernel is unique. Therefore, our problem reduces to finding an appropriate kernel.

\begin{theorem}[Moore-Aronszajn Theorem~\cite{berlinet2004reproducing}]
\label{th:ma}
Let $k$ be a positive definite function on $\Xcal \times \Xcal$. There exists only one Hilbert space $\Hcal$ of functions on $\Xcal$ with $k$ as reproducing kernel. The subspace $\Hcal_0$ of $\Hcal$ spanned by the function $k(\cdot, {x}), {x} \in \Xcal$ is dense~\footnote{A dense subset of $M$ implies the closure of the subset $X$ equals $M$.} in $\Hcal$ and $\Hcal$ is the set of functions on $\Xcal$ which are point-wise limits of Cauchy sequence in $\Hcal_0$ with the inner product 
\begin{equation}
\label{eq:iprkhs}
	\langle f, g \rangle_{\Hcal_0} = \sum_{i=1}^n \sum_{j=1}^m \alpha_i \beta_j k({z}_j, {x}_i),
\end{equation}
where $f =  \sum_{i=1}^n \alpha_i k(\cdot, {x}_i)$ and $g =  \sum_{j=1}^m \beta_j k(\cdot, {z}_j)$.
\end{theorem}

The important property while working in an RKHS is that the convergence in norm implies point-wise convergence; the converse need not be true. In other words, if two functions in an RKHS are close in the norm sense, they are also close point-wise. We will rely on this property to solve the problem discussed in this paper. In Theorem~\ref{th:ma}, $f$ and $g$ are defined only in $\Hcal_0$. The following theorem known as the representer theorem ensures that the solution of minimizing the regularized risk functional admits such a representation.

\begin{theorem}[Nonparametric Representer Theorem~\cite{scholkopf2001generalized}]\label{th:representer}
Let $\Xcal$ be a nonempty set and $\Hcal$ be an RKHS with reproducing kernel $k$ on $\Xcal \times \Xcal$. Suppose we are given a training sample $({x}_1,y_1),\dots,({x}_m,y_m) \in \Xcal \times \mathbb{R}$, a strictly monotonically increasing real-valued function $h$ on $[0, \infty)$, an arbitrary cost function $c:(\Xcal \times \mathbb{R}^2)^m \to \mathbb{R} \cup \{\infty\}$, and a class of functions~\footnote{$\mathbb{R}^{\Xcal}$ is the space of functions mapping $\Xcal$ to $\mathbb{R}$.}
\begin{align}
\resizebox{\columnwidth}{!}{$
\nonumber \Fcal = \{f \in \mathbb{R}^{\Xcal} | f(\cdot) = \sum_{i=1}^\infty \beta_i k(\cdot, {z}_i), \beta_i \in \mathbb{R}, {z}_i \in \Xcal, \lVert f \rVert_{\Hcal_{k}} < \infty \}.$}
\end{align}
Then any $f \in \Fcal$ minimizing the regularized risk functional 
\begin{equation}
\nonumber c(({x}_1,y_1,f(x_1)),\dots,({x}_m,y_m,f({x}_m))) + h(\lVert f \rVert_{\Hcal_{k}})
\end{equation}
admits a representation of the form
\begin{equation}
	f(\cdot) =  \sum_{i=1}^m \alpha_i k(\cdot, {x}_i).
\end{equation}
\end{theorem}
\section{Problem Setup}
\label{sec:problem}

Let $M$ be a smooth manifold and consider two (finite) collections of points, $X=\{x_i\}$, $Z=\{z_j\}\subset M$. Also, suppose we have a (Lie) group, $\Gcal$, acting on $M$. We want to determine which element $h\in \Gcal$ aligns the two point clouds $X$ and $hZ = \{hz_j\}$ the ``best.'' To assist with this, we will assume that each point contains information described by a point in an inner product space, $(\mathcal{I},\langle\cdot,\cdot\rangle_{\mathcal{I}})$. To this end, we will introduce two labeling functions, $\ell_X:X\to\Ical$ and $\ell_Z:Z\to\Ical$.

In order to measure their alignment, we will be turning the clouds, $X$ and $Z$, into functions $f_X,f_Z:M\to\Ical$ that live in some reproducing kernel Hilbert space, $(\Hcal,\langle\cdot,\cdot\rangle_{\mathcal{H}})$.
\begin{remark}
	The action, $\Gcal\curvearrowright M$ induces an action $\Gcal\curvearrowright C^\infty(M)$ by
	$$h.f(x) := f(h^{-1}x).$$
	Inspired by this observation, we will set $h.f_Z := f_{h^{-1}Z}$.
\end{remark}
\begin{problem}\label{prob:problem}
	The problem of aligning the point clouds can now be rephrased as maximizing the scalar products of $f_X$ and $h.f_Z$. i.e. We want to solve
	\begin{equation}\label{eq:max}
		\arg\max_{h\in \Gcal} \, F(h),\quad F(h):= \langle f_X, h.f_Z\rangle_{\mathcal{H}}.
	\end{equation}
\end{problem}
\subsection{Constructing the functions}
We first choose a symmetric function $k:M\times M\to \mathbb{R}$ to be the kernel of our RKHS, $\Hcal$. This allows us to turn the point clouds to functions via
\begin{equation}
\begin{split}
	f_X(\cdot) &:= \sum_{x_i\in X} \, \ell_X(x_i) k(\cdot,x_i), \\
	f_Z(\cdot) &:= \sum_{z_j\in Z} \, \ell_Z(z_j) k(\cdot,z_j).
\end{split}
\end{equation}
We can now define the inner product of $f_X$ and $f_Z$ by
\begin{equation}\label{eq:scalar}
\langle f_X,f_Z\rangle_{\Hcal} := \sum_{\substack{x_i\in X\\ z_j\in Z}} \, \langle \ell_X(x_i),\ell_Z(z_j)\rangle_{\mathcal{I}} \cdot k(x_i,z_j).
\end{equation}
\begin{remark}
	We note two advantages of measuring the alignment of $X$ and $Z$ by \eqref{eq:scalar}. The first is that we do not need identification of which point of $X$ should be paired with what point of $Z$. The second is that the number of points in $X$ does not even need to be equal to the number of points in~$Z$!
\end{remark}
\subsection{Building the Gradient Flow}
In order to (at least locally) solve~\eqref{eq:max}, we will construct a gradient flow: $\dot{h}=\nabla F(h)$. Before we can do this, we will first determine the differential, $dF$. In order to do this, we will need the notion of an infinitesimal generator for a group action (see chapter 4 of \cite{berndt2001introduction}). 
\begin{definition}
	Suppose that a Lie Group $\Gcal$ acts diffeomorphically on a smooth manifold $M$ via $\varphi$; that is
	\begin{equation}
	\begin{split}
		\varphi:\Gcal&\to \diff(M) \\
		g &\mapsto \varphi_g.
	\end{split}
	\end{equation}
	For a given $\xi\in\mathfrak{g}=\Lie(\Gcal)$, we denote the vector field $\xi_M$ (called the infinitesimal generator) on $M$ given by the rule:
	\begin{equation}
		df_x(\xi_M) := \left.\frac{d}{dt}\right|_{t=0} \, f(\varphi_{\exp(t\xi)}(x)),\quad f\in C^\infty(M).
	\end{equation}
\end{definition}
This lets us compute the differential, $dF$.

\begin{remark}
To make notation more concise for the remainder of this paper, we will denote $c_{ij} := \langle \ell_X(x_i),\ell_Z(z_j)\rangle_{\Ical}$.
\end{remark}

\begin{theorem}
	Suppose that $F(h) = \langle f_X,h.f_Z\rangle_{\mathcal{H}}$ as described above. Then
	\begin{equation}\label{eq:differential}
	dF_e(\xi) = \sum_{\substack{x_i\in X\\ z_j\in Z}} \, c_{ij}\cdot d\left(\tilde{k}_{x_i}\right)_{z_j}\left(-\xi_M(z_j)\right),
	\end{equation}
	where $\tilde{k}_{x_i} = k(x_i,\cdot)$.
\end{theorem}
\begin{remark}
	The notation for the differential of a function used throughout this paper is $df_x(v)$, where $x\in M$ and $v\in T_xM$:
	\begin{equation}
	df_x(v) = \left.\frac{d}{dt}\right|_{t=0} \, f(c(t)),\quad c(0)=x,\quad c'(0)=v.
	\end{equation}
\end{remark}
\begin{proof}
	This follows from a straight-forward application of the chain rule. 
	\begin{equation}
	\begin{split}
	dF_e(\xi) &= \left.\frac{d}{dt}\right|_{t=0} \, \langle f_X,\exp(t\xi).f_Z\rangle_{\Hcal} \\
	&= \sum \, c_{ij}\cdot \left.\frac{d}{dt}\right|_{t=0}\, k\left(x_i,\exp(-t\xi)z_j\right) \\
	&= \sum \, c_{ij}\cdot d\left(\tilde{k}_{x_i}\right)_{z_j} \cdot \left.\frac{d}{dt}\right|_{t=0} \, \exp(-t\xi)z_j \\
	&= \sum_{\substack{x_i\in X\\ z_j\in Z}} \, c_{ij} \cdot d\left(\tilde{k}_{x_i}\right)_{z_j}\left(-\xi_M(z_j)\right).
	\end{split}
	\end{equation}
	Which matches equation \eqref{eq:differential}.
\end{proof}
Of course, to construct the gradient flow we are interested in computing $dF_h$ instead of just $dF_e$. We will accomplish this via left-translation. Left-translation is given by the smooth map $\ell_h:\Gcal\to \Gcal$ where $x\mapsto hx$. Its differential gives rise to an isomorphism of tangent spaces, $(\ell_h)_*:\mathfrak{g}\xrightarrow{\sim} T_h\Gcal$.
\begin{corollary}
	Under the identification $T_h\Gcal \cong (\ell_h)_*\mathfrak{g}$, we have
	\begin{equation}
	\begin{split}
	dF_h&\left( (\ell_h)_*\xi\right) = \\
	&=
	\sum_{\substack{x_i\in X\\ z_j\in Z}} \, c_{ij} \cdot
	d\left(\tilde{k}_{x_i}\right)_{h^{-1}z_j}\left(-\xi_M(h^{-1}z_j)\right),
	\end{split}
	\end{equation}
	where $c_{ij} = \langle \ell_X(x_i),\ell_Z(z_j)\rangle_{\Ical}$.
\end{corollary}

In order to turn the co-vector $dF_h\in T_h^*\Gcal$ into a vector $\nabla F_h\in T_h\Gcal$, we will use a left-invariant metric. This can be accomplished by defining an inner-product, $\langle\cdot,\cdot\rangle_\mathfrak{g}$ on $\mathfrak{g}$ and lifting to a (Riemannian) metric on $\Gcal$ via left-translation, i.e.
$$\langle (\ell_h)_*\eta,(\ell_h)_*\xi\rangle_{T_h\Gcal} := \langle \eta,\xi\rangle_{\mathfrak{g}}.$$
This allows us to define the gradient of $F$ as
\begin{equation}
\langle \nabla F_h,(\ell_h)_*\xi\rangle_{T_h\Gcal} = dF_h\left((\ell_h)_*\xi\right).
\end{equation}
This allows for a way to obtain a (local) solution to \eqref{eq:max} by following
\begin{equation}\label{eq:gradient_flow}
\dot{h} = \nabla F(h).
\end{equation}

\section{Special Euclidean Group}\label{sec:euclidean}
We will specialize the above treatment for the case where $\Gcal = \SE(n)$, the special Euclidean group in $n$ dimensions. We will likewise let $M=\mathbb{R}^n$ on which $\SE(n)$ acts in the usual way: let $(R,T)\in \SE(n)$ where $R\in \SO(n)$ and $T\in\mathbb{R}^n$,
$$(R,T).x = Rx+T,\quad x\in\mathbb{R}^n.$$
We will also choose the squared exponential kernel for $k:\mathbb{R}^n\times\mathbb{R}^n\to\mathbb{R}$:
\begin{equation}
k(x,y) = \sigma^2\exp\left(\frac{-\lVert x-y\rVert_n^2}{2\ell^2}\right),
\end{equation}
for some fixed real parameters $\sigma$ and $\ell$, and $\lVert\cdot\rVert_n$ is the standard Euclidean norm on $\mathbb{R}^n$. In order to determine the gradient flow \eqref{eq:gradient_flow}, we need to compute the infinitesimal generators of the action $\SE(n)\curvearrowright\mathbb{R}^n$ as well as decide on a left-invariant metric for $\se(n)=\Lie(\SE(n))$.
\subsection{Infinitesimal Generator}
For a fixed $\xi\in\se(n)$, it has the form $\xi=(\omega,v)$ where $\omega\in\so(n)$ and $v\in\mathbb{R}^n$. Because the infinitesimal generator map $\mathfrak{g}\to \mathfrak{X}(M)$ is a Lie algebra homomorphism (where $\mathfrak{X}(M)$ is the space of all vector fields over $M$, see \textsection 27 of \cite{diffGeometry}), we see that $\xi_M = \omega_M+v_M$. A straight forward computation leads to
\begin{equation}
\xi_{\mathbb{R}^n}x = \hat{\omega}x+v,\quad \hat{\omega}\in \emph{Skew}(n),\quad v\in\mathbb{R}^n.
\end{equation}
\subsection{Metric}
We need to choose a metric on $\se(n)$ to turn $dF_h$ into $\nabla F_h$. For this example, we will take a multiple of the Killing form on $\so(n)$ and the Euclidean norm on $\mathbb{R}^n$. That is,
\begin{equation}\label{eq:metric}
\langle (\omega,v),(\eta,u)\rangle_{\se(n)} =  b^2 \cdot\langle v,u\rangle_n - a^2\left(\frac{n-2}{2}\right)\cdot \tr(\omega\eta),
\end{equation}
where $\langle\cdot,\cdot\rangle_n$ is the standard Euclidean inner product on $\mathbb{R}^n$ (see \citep{park1994kinematic} for a discussion in three dimensions), as well as $a$ and $b$ are tuning parameters. The reason for the $(2-n)/2$ term is because with this normalization (with $a=1$) the skew matrices $E_{ij}-E_{ji}$ are orthonormal. Here $E_{ij}$ denotes the matrix with only zeros except for a $1$ in the $(i,j)$-coordinate.
\subsection{Calculating the Gradient}
Before we find the gradient, let us first determine its differential (at the identity for simplicity). 
\begin{equation}
\resizebox{\columnwidth}{!}{$
\begin{split}
dF_e(\xi) &= \sum_{\substack{x_i\in X\\ z_j\in Z}} \, c_{ij} \cdot d\left(\tilde{k}_{x_i}\right)_{z_j}\left(-\xi_M(z_j)\right) \\
&= \sum_{\substack{x_i\in X\\ z_j\in Z}} \, c_{ij} \cdot 
\frac{1}{\ell^2}k(x_i,z_j)\cdot \langle (z_j-x_i) , \left( -\hat{\omega}z_j-v\right)\rangle_n.
\end{split}$}
\end{equation}
To turn $dF_e$ into $\nabla F_e$, we will compute $\nabla_\omega F_e$ and $\nabla_v F_e$ separately:
\begin{equation}
\begin{split}
-a^2&\left(\frac{n-2}{2}\right)\cdot\tr\left[(\nabla_\omega F_e)\hat{\omega}\right] = \\
&\sum_{\substack{x_i\in X\\ z_j\in Z}} \, c_{ij} \cdot 
\frac{1}{\ell^2}k(x_i,z_j)\cdot \langle (z_j-x_i) , \left( -\hat{\omega}z_j\right)\rangle_n,\\
b^2&\langle (\nabla_vF_e),v\rangle_n =\\
&\sum_{\substack{x_i\in X\\ z_j\in Z}} \, c_{ij} \cdot 
\frac{1}{\ell^2}k(x_i,z_j)\cdot \langle (z_j-x_i) , \left( -v\right)\rangle_n.
\end{split}
\end{equation}
To solve for this in coordinates, we will let $\{e^m\}_{m=1}^n$ be the standard orthonormal basis for $(\mathbb{R}^n,\langle\cdot,\cdot\rangle_n)$ and $\{ J^{pq} \}_{p<q} := \{E_{pq}-E_{qp}\}_{p<q}$ be as above. Then, the gradient becomes:
\begin{equation}
\resizebox{\columnwidth}{!}{$
\begin{split}
\left(\nabla_\omega F_e\right)^{pq} &= \frac{1}{a^2\ell^2} \sum_{\substack{x_i\in X\\ z_j\in Z}} \, c_{ij} \cdot 
k(x_i,z_j)\cdot \langle (z_j-x_i) , \left( -J^{pq}z_j\right)\rangle_n, \\
\left( \nabla_v F_e\right)^m &= \frac{1}{b^2\ell^2} \sum_{\substack{x_i\in X\\ z_j\in Z}} \, c_{ij} \cdot 
k(x_i,z_j)\cdot \langle (z_j-x_i) , \left( -e^m\right)\rangle_n.
\end{split}$}
\end{equation}
The above can be simplified by computing the inner product on the right hand side:
\begin{equation}
\begin{split}
\left(\nabla_\omega F_e\right)^{pq} &= \frac{1}{a^2\ell^2} \sum_{\substack{x_i\in X\\ z_j\in Z}} \, c_{ij} \cdot 
k(x_i,z_j)\cdot \left(x_i^pz_j^q-x_i^qz_j^p\right), \\
\left( \nabla_v F_e\right)^m &= \frac{1}{b^2\ell^2} \sum_{\substack{x_i\in X\\ z_j\in Z}} \, c_{ij} \cdot 
k(x_i,z_j)\cdot \left(x_i^m-z_j^m\right).
\end{split}
\end{equation}
Likewise, to translate away from the origin, we note that if $h=(R,T)\in\SE(n)$, $(\ell_h)_*(\hat{\omega},v) = (R\hat{\omega},Rv)$. Then, if we express the gradient as $\nabla F_h = (\ell_h)_*(\hat{\omega},v) = (R\hat{\omega},Rv)$, we get the following expression for $(\hat{\omega},v)\in\se(n)$:
\begin{equation}\label{eq:flow}
\begin{split}
\hat{\omega}^{pq} &= \frac{1}{a^2\ell^2} \sum_{\substack{x_i\in X\\ z_j\in Z}} \, c_{ij} \cdot 
k(x_i,\tilde{z}_j)\cdot \left(x_i^p\tilde{z}_j^q-x_i^q\tilde{z}_j^p\right), \\
 v^m &= \frac{1}{b^2\ell^2} \sum_{\substack{x_i\in X\\ z_j\in Z}} \, c_{ij} \cdot 
k(x_i,\tilde{z}_j)\cdot \left(x_i^m-\tilde{z}_j^m\right),
\end{split}
\end{equation}
where $\tilde{z}_j = h^{-1}z_j = R^{\transpose} z_j-R^{\transpose} T$.
\section{Analysis and Verification of Idea}
\label{sec:verification}
It is important to take a moment and examine when solving Problem \ref{prob:problem} actually causes the clouds to be best aligned. It is of course impossible to perfectly align two non-identical clouds. Presented below is a discussion of when the two clouds \textit{are} identical, when does the identity element in the group maximize \eqref{eq:max}?

Suppose that $Z=X$ and $\ell_Z=\ell_X$. Then the identity ideally should be a fixed-point of \eqref{eq:gradient_flow}. This leads to the following:
\begin{theorem}
	Assume that for all $h\in \Gcal$ and $x\in M$, $k(hx,hx)\leq k(x,x)$. Then the identity is a global maximum of $F$.
\end{theorem}
\begin{proof}
	We have that
	\begin{equation}
	F(h) = \langle f_X,h.f_X\rangle_{\Hcal},\quad F(e) = \lVert f_X\rVert_{\Hcal}^2 \geq 0.
	\end{equation}
	Then using the Cauchy-Schwarz inequality we obtain:
	\begin{equation}
	\begin{split}
		\left|\langle f_X,h.f_X\rangle_{\Hcal}\right| &\leq \lVert f_X\rVert_{\Hcal} \cdot \lVert h.f_X\rVert_{\Hcal},
	\end{split}
	\end{equation}
	which is less than $F(e)$ provided that $\lVert h.f_X\rVert_{\Hcal}\leq \lVert f_X\rVert_{\Hcal}$. Computing this, we see that
	\begin{equation}
	\begin{split}
		\lVert h.f_X\rVert_{\Hcal}^2 &= \sum_{x_i,z_j\in X} \, c_{ij}\cdot k\left(h^{-1}x_i,h^{-1}z_j\right)\\
		&\leq \sum_{x_i,z_j\in X} \, c_{ij}\cdot k\left(x_i,z_j\right) = \lVert f_X\rVert_{\Hcal}^2.
	\end{split}
	\end{equation}
	Combining everything, we get that
	\begin{equation}
	\left|F(h)\right| \leq \lVert f_X\rVert_{\Hcal}\cdot \lVert h.f_X\rVert_{\Hcal} \leq \lVert f_X\rVert_{\Hcal}^2 = F(e).
	\end{equation}
\end{proof}
\begin{corollary}
	Suppose $k:M\times M\to\mathbb{R}$ is a stationary kernel~\citep[Page 82]{rasmussen2006gaussian}, that is $k(x,y) = k(d(x,y))$ for some distance function $d$. If $\Gcal$ acts isometrically on $M$, then the identity is a global maximum of $F$.
\end{corollary}
\begin{corollary}
	The identity is a global maximum for the $\SE(n)$ case.
\end{corollary}
\begin{proof}
	This follows from the fact that $\SE(n)$ acts on $\mathbb{R}^n$ isometrically, i.e. $\lVert hx-hy\rVert_n = \lVert x-y\rVert_n$.
\end{proof}

\begin{theorem}
        The maximizer of Problem~\ref{prob:problem}, minimizes the angle between $f_X$ and $f_Z$.
\end{theorem}
\begin{proof}
        Suppose $h^* \in \Gcal$ is the maximizer of Problem~\ref{prob:problem}. Then $\langle f_X, f_Z^* \rangle \geq \langle f_X, f_Z \rangle$ and $\lVert f_Z^* \rVert_{\Hcal} \leq \lVert f_Z \rVert_{\Hcal}$. Using Cauchy-Schwarz inequality:
    \begin{equation}
    \resizebox{\hsize}{!}{$
        \nonumber 0 \leq \lvert \langle f_X, f_Z \rangle \rvert \leq \lvert \langle f_X, f_Z^* \rangle \rvert \leq \lVert f_X \rVert_{\Hcal} \lVert f_Z^* \rVert_{\Hcal} \leq \lVert f_X \rVert_{\Hcal} \lVert f_Z \rVert_{\Hcal}
        $}
    \end{equation}
    dividing by $\lVert f_X \rVert_{\Hcal} \lVert f_Z \rVert_{\Hcal}$ and replacing $\lVert f_Z \rVert_{\Hcal}$ in the denominator by $\lVert f_Z^* \rVert_{\Hcal}$:
    \begin{equation}
        \nonumber 0 \leq \cos(\theta) \leq \frac{\lvert \langle f_X, f_Z^* \rangle \rvert}{\lVert f_X \rVert_{\Hcal} \lVert f_Z^* \rVert_{\Hcal}} \leq \frac{\lVert f_Z^*\rVert_{\Hcal}}{\lVert f_Z^* \rVert_{\Hcal}} \leq 1
    \end{equation} 
    \begin{equation}
        \nonumber 0 \leq \cos(\theta) \leq \cos(\theta^*) \leq 1
    \end{equation}
    \begin{equation}
        \nonumber 0 \leq \theta^* \leq \theta \leq \frac{\pi}{2}
    \end{equation}
where $\cos(\theta) = \lvert \langle f_X, f_Z \rangle \rvert / (\lVert f_X \rVert_{\Hcal} \lVert f_Z \rVert_{\Hcal})$.
\end{proof}
\section{Integrating the Flow for The Special Euclidean Group}
\label{sec:senintegration}

Now that we know the direction for the flow, what remains is to determine a way to integrate the flow and to determine a reasonable step size. We integrate using the Lie exponential map to preserve the group structure and the step size is calculated using a $4^{th}$-order Taylor approximation in a line search algorithm.
\subsection{Integrating}
We will take care that in integrating \eqref{eq:gradient_flow}, our trajectories will remain on $SE(n)$. This is slightly problematic because integrating is an additive process and $SE(n)$ is not closed under addition. To address this, we note that if in \eqref{eq:gradient_flow}, $(\hat{\omega},v)$ is constant in $\se(n)$ (i.e. $\nabla F_h$ is a left-invariant vector field) the solution is merely
\begin{equation}
(R(t),T(t)) = (R_0,T_0)\exp(t(\hat{\omega},v)),
\end{equation}
where $\exp:\se(n)\to\SE(n)$ is the Lie exponential map (which is merely the matrix exponential). We will exploit this by assuming that $\hat{\omega}$ and $v$ are constant over each time step.
\begin{align}
\begin{bmatrix}
R_{k+1} & T_{k+1} \\ 0 & 1
\end{bmatrix} &= \begin{bmatrix}
R_{k} & T_{k} \\ 0 & 1
\end{bmatrix} \cdot \exp \begin{bmatrix}
t\hat{\omega} & t v \\ 0 & 0
\end{bmatrix} \\
&= \begin{bmatrix}
R_k & T_k \\ 0 & 1
\end{bmatrix}\begin{bmatrix}
\Delta R & \Delta T \\ 0 & 1
\end{bmatrix} \\
&= \begin{bmatrix}
R_k\Delta R & R_k\Delta T + T_k \\
0 & 1
\end{bmatrix}
\end{align}
Where explicit formulas for $\Delta R$ and $\Delta T$ will be discussed in \textsection\ref{subsec:se3} for the special case where $n=3$.
Combining all of this, we get our integration step to be
\begin{align}\label{eq:newstep}
R_{k+1} &= R_k \Delta R \\
T_{k+1} &= R_k\Delta T + T_k.
\end{align}
\subsection{Step size}\label{sec:step} 
We can use \eqref{eq:flow} to point in  the direction of maximal growth and \eqref{eq:newstep} to find the updated element in $\SE(n)$. However, we currently have no intelligent way of choosing $t$. We will proceed by a Taylor approximation of the solution curve. If we let $G(t) := F(h\exp(t\xi))$, then we want to find the value of $t$ that maximizes $G$. We compute a $4^{th}$-order Taylor expansion of $G(t)$ about $t=0$ and determine the value of $t$ that maximizes this polynomial.
\begin{equation}\label{eq:fourth}\begin{split}
G(t) \approx \sumxz \, c_{ij} \cdot e^{\alpha_{ij}} \left\{ 
g^1_{ij}t + g^2_{ij}t^2 + g^3_{ij}t^3 + g^4_{ij}t^4 \right\},
\end{split}\end{equation}
where 
\begin{equation}
\begin{split}
g^1_{ij} &= \beta_{ij} \\
g^2_{ij} &= \gamma_{ij}+\frac{1}{2}\beta_{ij}^2 \\
g^3_{ij} &= \delta_{ij}+\beta_{ij}\gamma_{ij}+\frac{1}{6}\beta_{ij}^3 \\
g^4_{ij} &= \varepsilon_{ij} + \beta_{ij}\delta_{ij}+\frac{1}{2}\beta_{ij}^2\gamma_{ij} + \frac{1}{2}\gamma_{ij}^2 + \frac{1}{24}\beta_{ij}^4 \\
\alpha_{ij} &= \frac{-1}{2\ell^2}\lVert x_i-z_j\rVert_n \\
\beta_{ij} &= \frac{-1}{\ell^2} \langle \hat{\omega}z_j+v,x_i-z_j\rangle_n \\
\gamma_{ij} &= \frac{-1}{2\ell^2}\left( \lVert \hat{\omega}z_j+v\rVert_n^2 + 2\langle \hat{\omega}^2z_j+v,x_i-z_j\rangle_n\right) \\
\delta_{ij} &= \frac{1}{\ell^2}\left( \langle -\hat{\omega}z_j-v,\hat{\omega}^2z_j +\hat{\omega}v\rangle_n \right. \\
&\left.\quad\quad\quad+ \langle -\hat{\omega}^3z_j-\hat{\omega}^2v,x_i-z_j\rangle_n\right) \\
\varepsilon_{ij} &= \frac{-1}{2\ell^2}\left( \lVert\hat{\omega}^2z_j+\hat{\omega}v\rVert_n^2 +2\langle \hat{\omega}z_j+v,\hat{\omega}^3z_j +\hat{\omega}^2v\rangle_n \right. \\
&\left. \quad\quad\quad+ 2\langle \hat{\omega}^4z_j+\hat{\omega}^3v,x_i-z_j\rangle_n\right)
\end{split}
\end{equation}
The ``optimal'' step size is then taken to be the value of $t>0$ that maximizes the quartic \eqref{eq:fourth}.


\subsection{Special Case: $n=3$}
\label{subsec:se3}

When we restrict attention to $\SE(3)$, the gradient \eqref{eq:flow} takes a special form:
\begin{equation}\label{eq:3_gradient}
\begin{split}
\omega = \frac{1}{a^2\ell^2} \, \sumxz c_{ij} k(x_i,h^{-1}z_j)\left(x_i\times(h^{-1}z_j)\right) \\
v = \frac{1}{b^2\ell^2}\, \sumxz c_{ij} k(x_i,h^{-1}z_j)\left( x_i - h^{-1}z_j\right).
\end{split}
\end{equation}
Additionally, an explicit formula for the exponential map $\exp:\se(3)\to\SE(3)$ exists (see \cite{1104.1106} and \cite{rohan2013} for example). This gives an exact way to solve \eqref{eq:newstep}.
\begin{equation}\label{eq:step}
\begin{split}
\Delta R &= I + \left( \frac{\sin t\theta}{\theta}\right)\hat{\omega} + 
\left(\frac{1-\cos t\theta}{\theta^2}\right)\hat{\omega}^2,\\
\Delta T &= \left[ t I + \left( \frac{1-\cos t\theta}{\theta^2}\right)\hat{\omega} + \left(
\frac{t\theta - \sin t\theta}{\theta^3}\right)\hat{\omega}^2
\right] v.
\end{split}
\end{equation}
where $\theta = \lVert \omega \rVert_3$ with $\omega\in\mathbb{R}^3$ and $t$ is taken to maximize $G(t)$ in equation \eqref{eq:fourth}.
\section{Evaluation and Discussion}
\label{sec:results}
First, we provide a numerical example to justify our choice of $4^{th}$-order Taylor expansion for the step length computation which in practice proven to be important. Next, we present the experimental evaluation using RGB-D data collected using the Microsoft Kinect sensor.    

\subsection{Taylor Expansion}
This subsection explains (numerically) the reason for taking the extra complexity in computing the 4$^{th}$-order approximation as opposed to a 2$^{nd}$-order approximation.
The following numerical example uses the point clouds supplied in 
MATLAB~\cite{cloudexample}. Here, $X$ is the first picture and $Z$ is the second. The parameter values used are $\sigma = 0.1$, $\ell = 0.15$, and $a^2 = b^2 = 7$.

With these chosen set of parameters, the computed gradient gives
$$\hat{\omega} = \begin{bmatrix}
0 &   3.6452 &  -38.0800\\
-3.6452   &      0 &  -1.2370\\
38.0800  &  1.2370  &       0
\end{bmatrix},\quad v = \begin{bmatrix}
-9.3792\\
1.5252\\
-8.9388
\end{bmatrix}.$$
Figure \ref{fig:comparisons} contains the plots comparing the exact function, $G(t)$, along with its $2^{nd}$- and $4^{th}$-order Taylor approximation.
\begin{figure}[t]
	\centering
	\includegraphics[width=0.7\columnwidth]{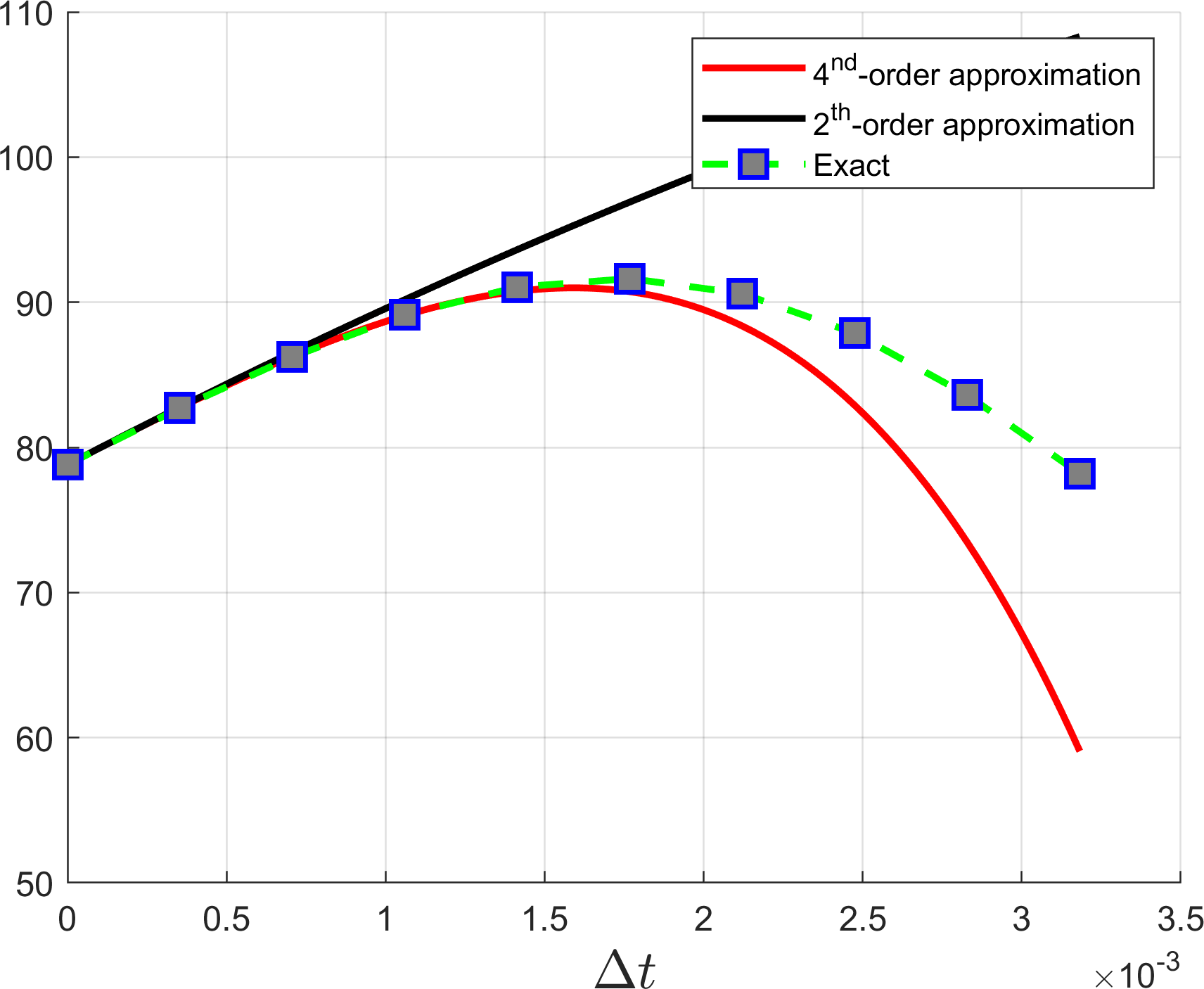}
	\caption{Comparing the exact values of $G(t)$ with its Taylor approximations. It is clear that the 2$^{nd}$-order approximation cannot capture the function behavior and, despite extra complexity, the approximation up to the 4$^{th}$-order is necessary to maintain the integration of the flow smooth.}
	\label{fig:comparisons}
	\squeezeup\squeezeup
\end{figure}
As can be seen from the figure, although the 2$^{nd}$-order approximation is a reasonable approximation for $t < 0.001$, it does not give a meaningful approximation for the maximum of $G(t)$. However, the 4$^{th}$-order does.

\begin{figure*}[t]
    \centering
    \subfloat{\includegraphics[width=0.5\columnwidth]{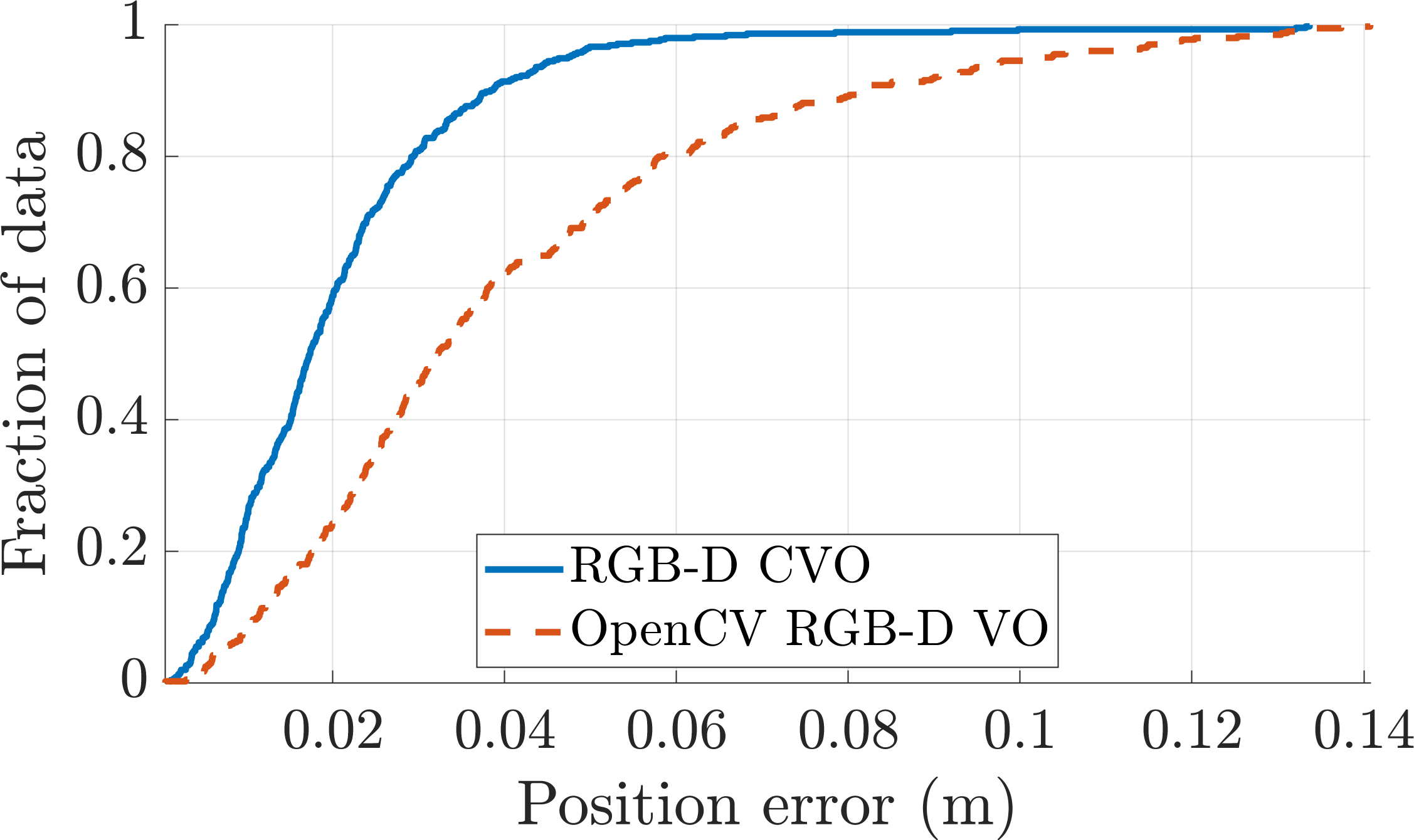}}
    \subfloat{\includegraphics[width=0.5\columnwidth]{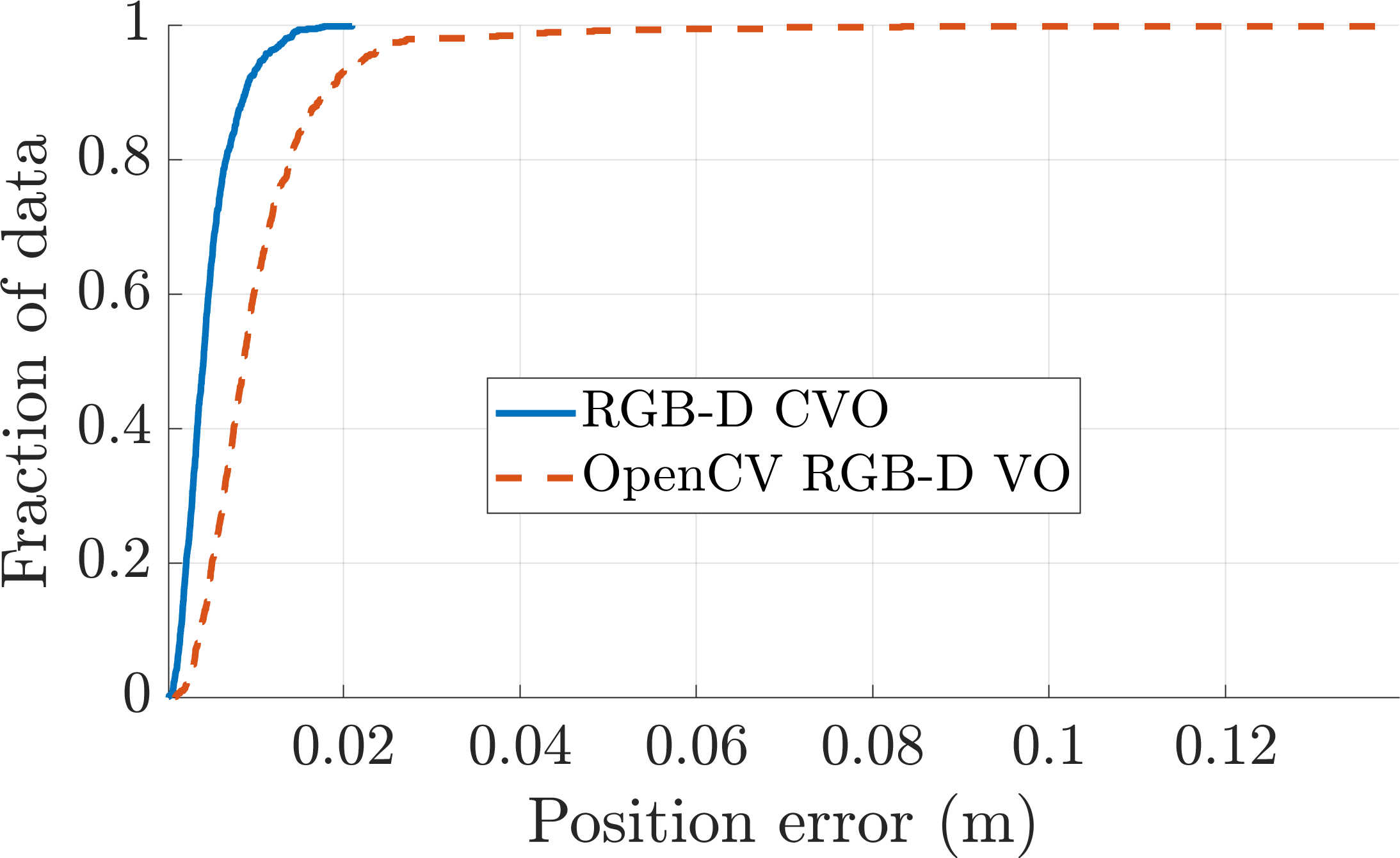}}
    \subfloat{\includegraphics[width=0.5\columnwidth]{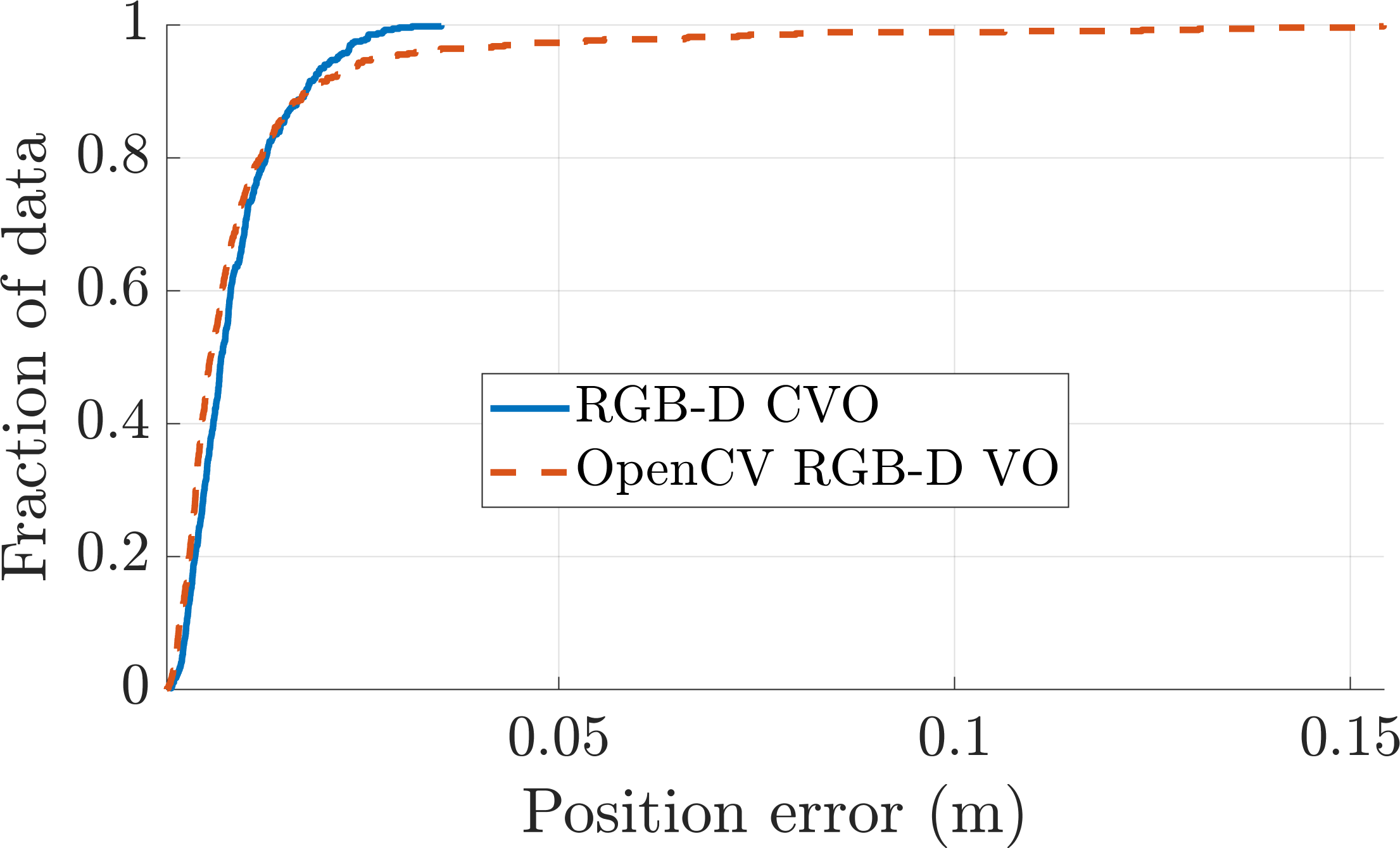}}
    \subfloat{\includegraphics[width=0.5\columnwidth]{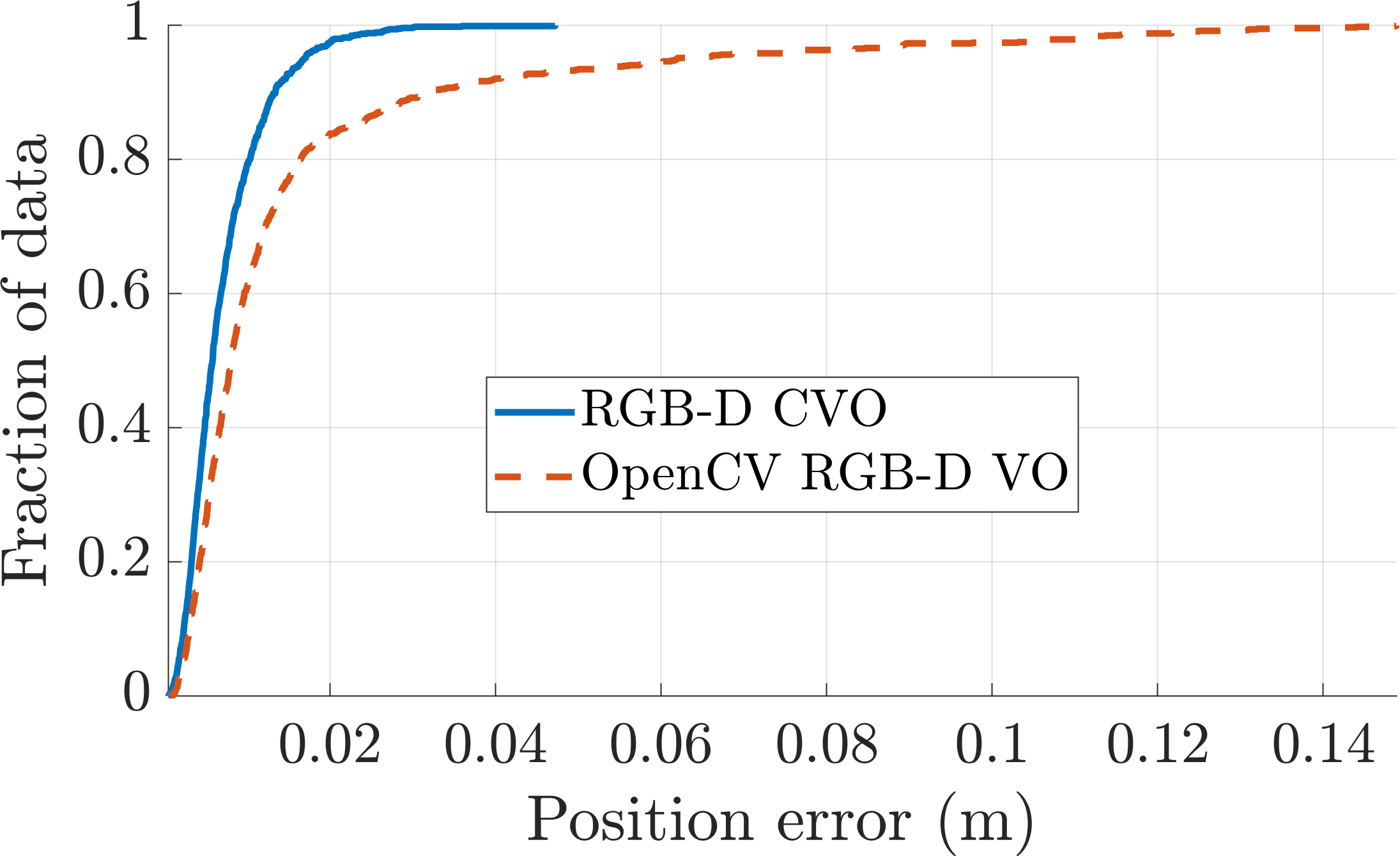}}\\
    \subfloat{\includegraphics[width=0.5\columnwidth]{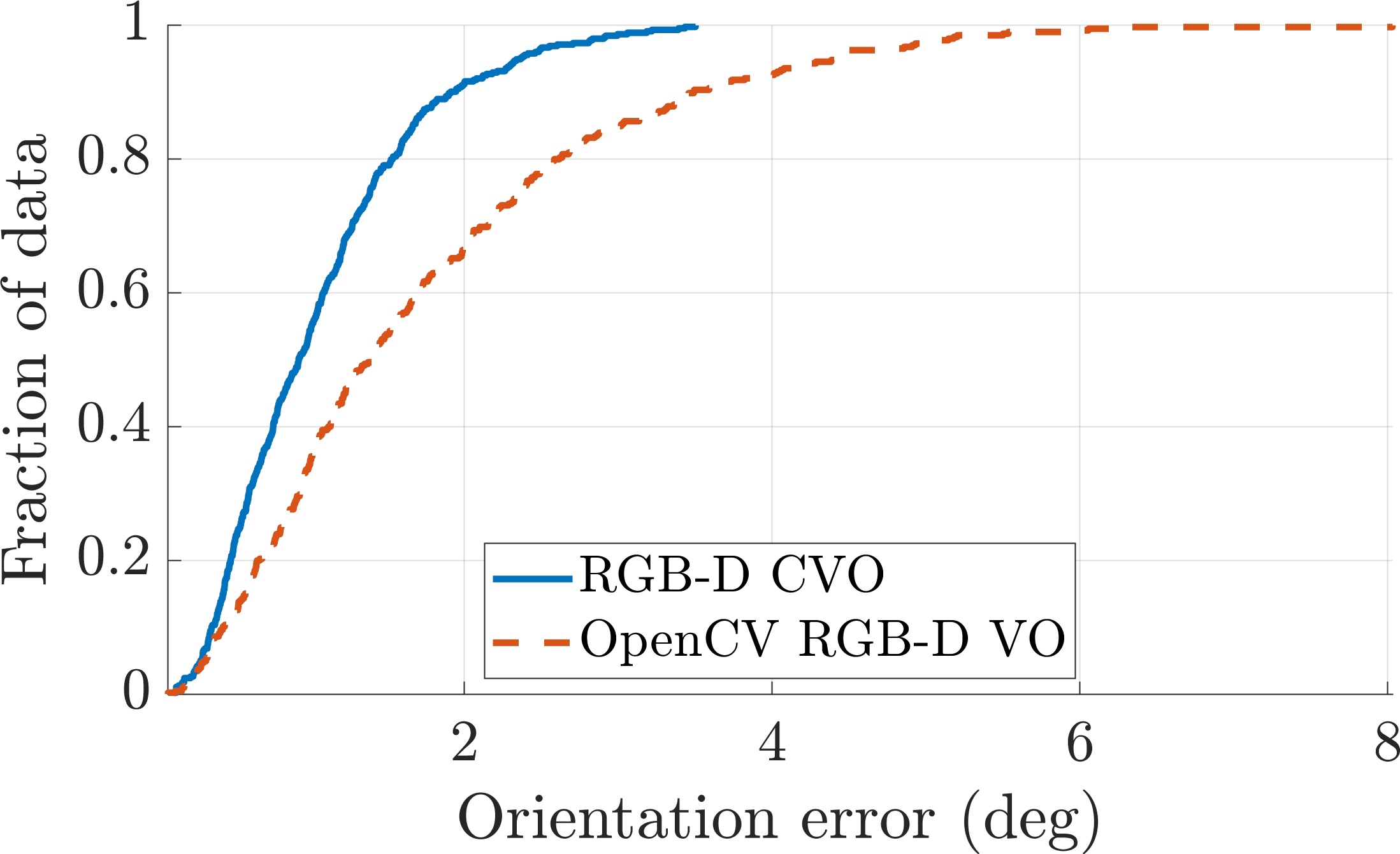}}
    \subfloat{\includegraphics[width=0.5\columnwidth]{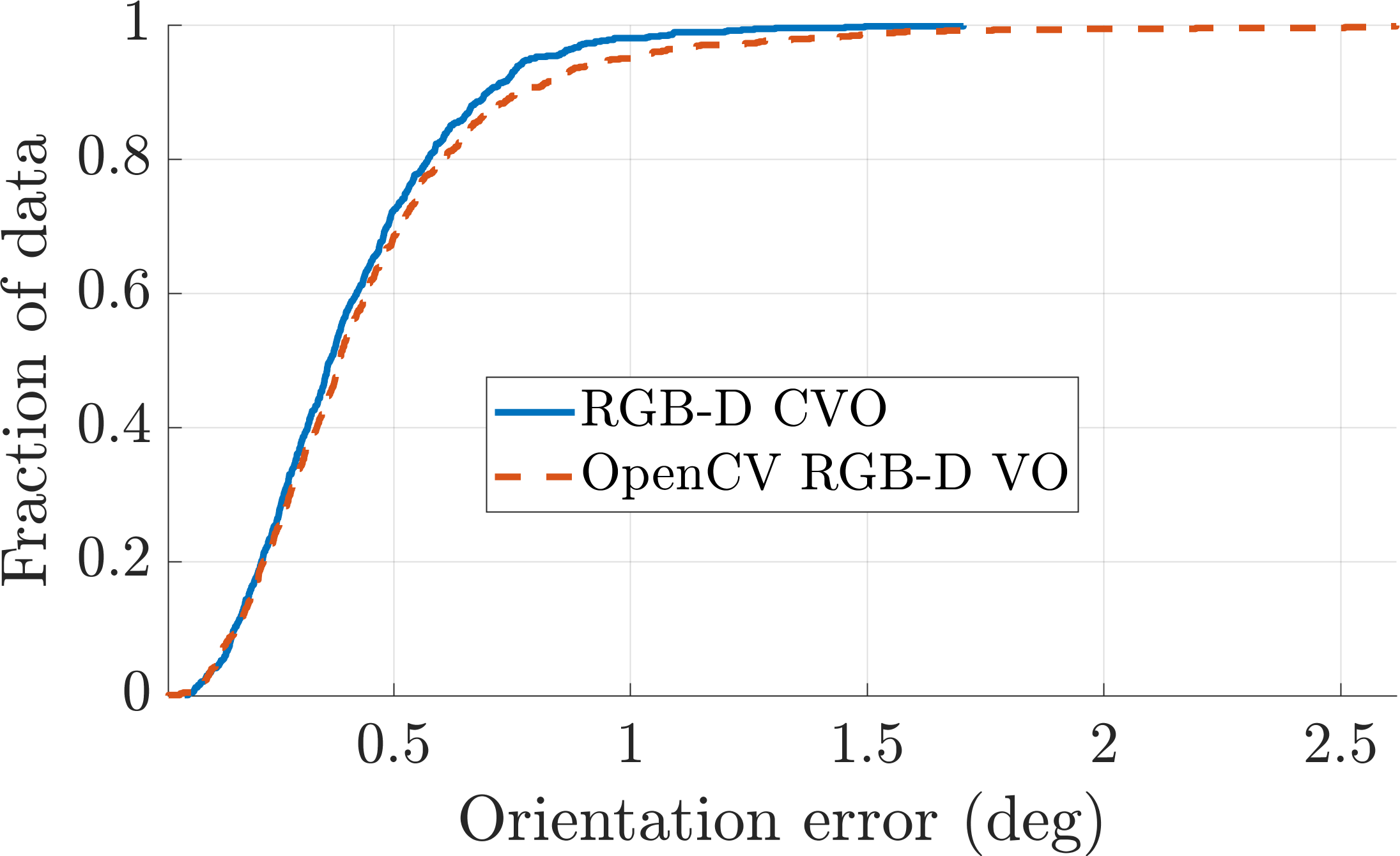}}
    \subfloat{\includegraphics[width=0.5\columnwidth]{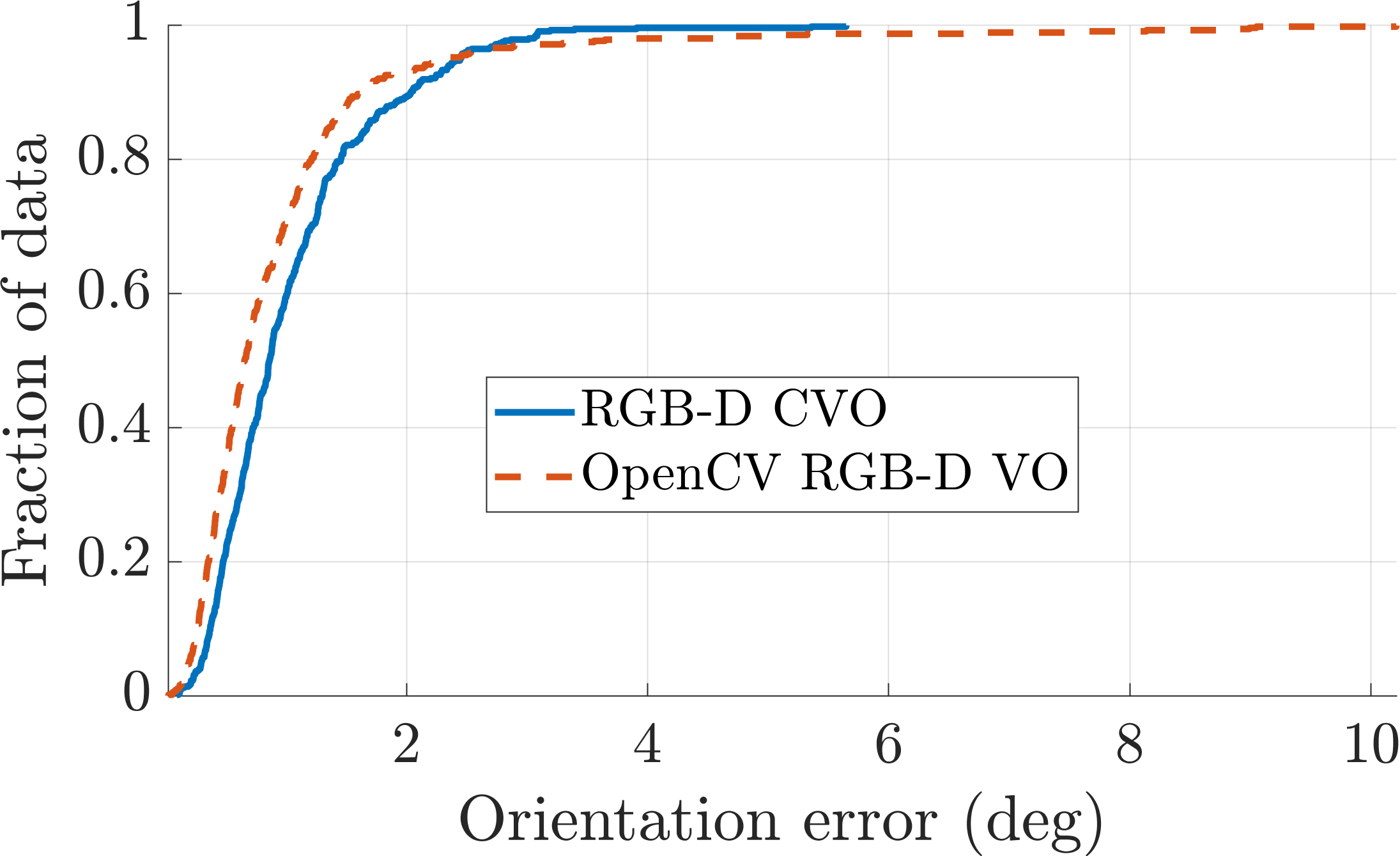}}
    \subfloat{\includegraphics[width=0.5\columnwidth]{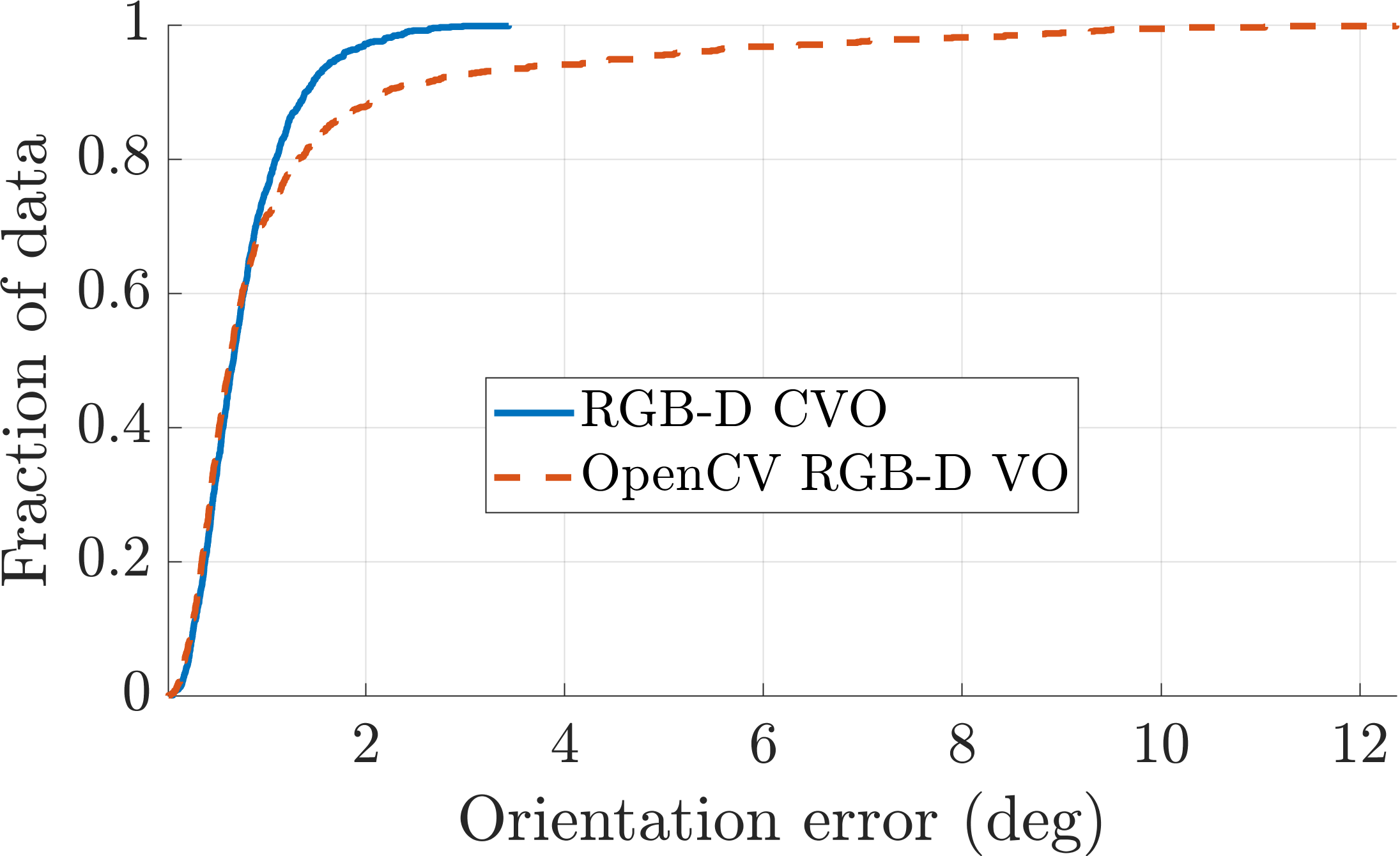}}\\
    \subfloat{\includegraphics[width=0.5\columnwidth]{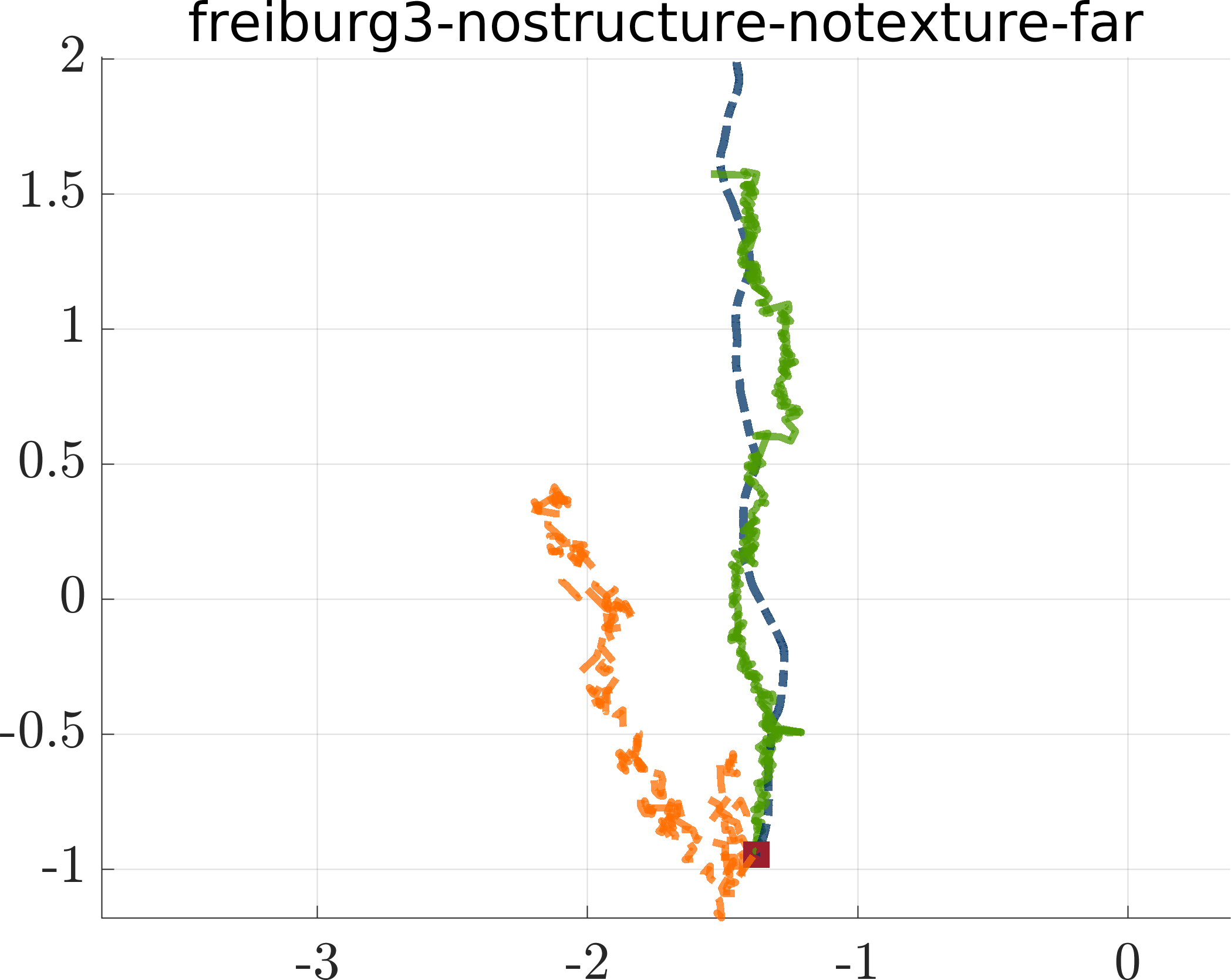}}
    \subfloat{\includegraphics[width=0.5\columnwidth]{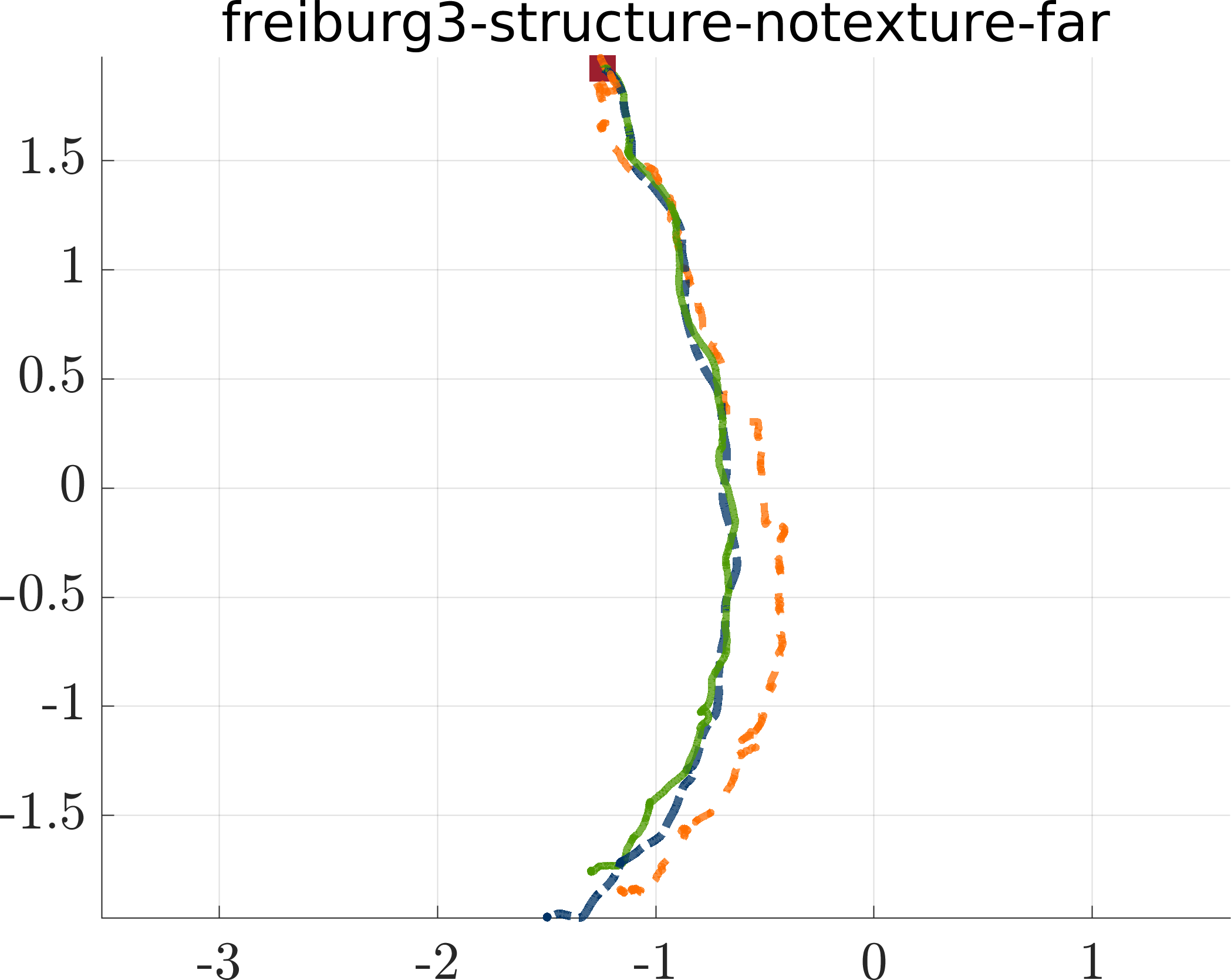}}
    \subfloat{\includegraphics[width=0.5\columnwidth]{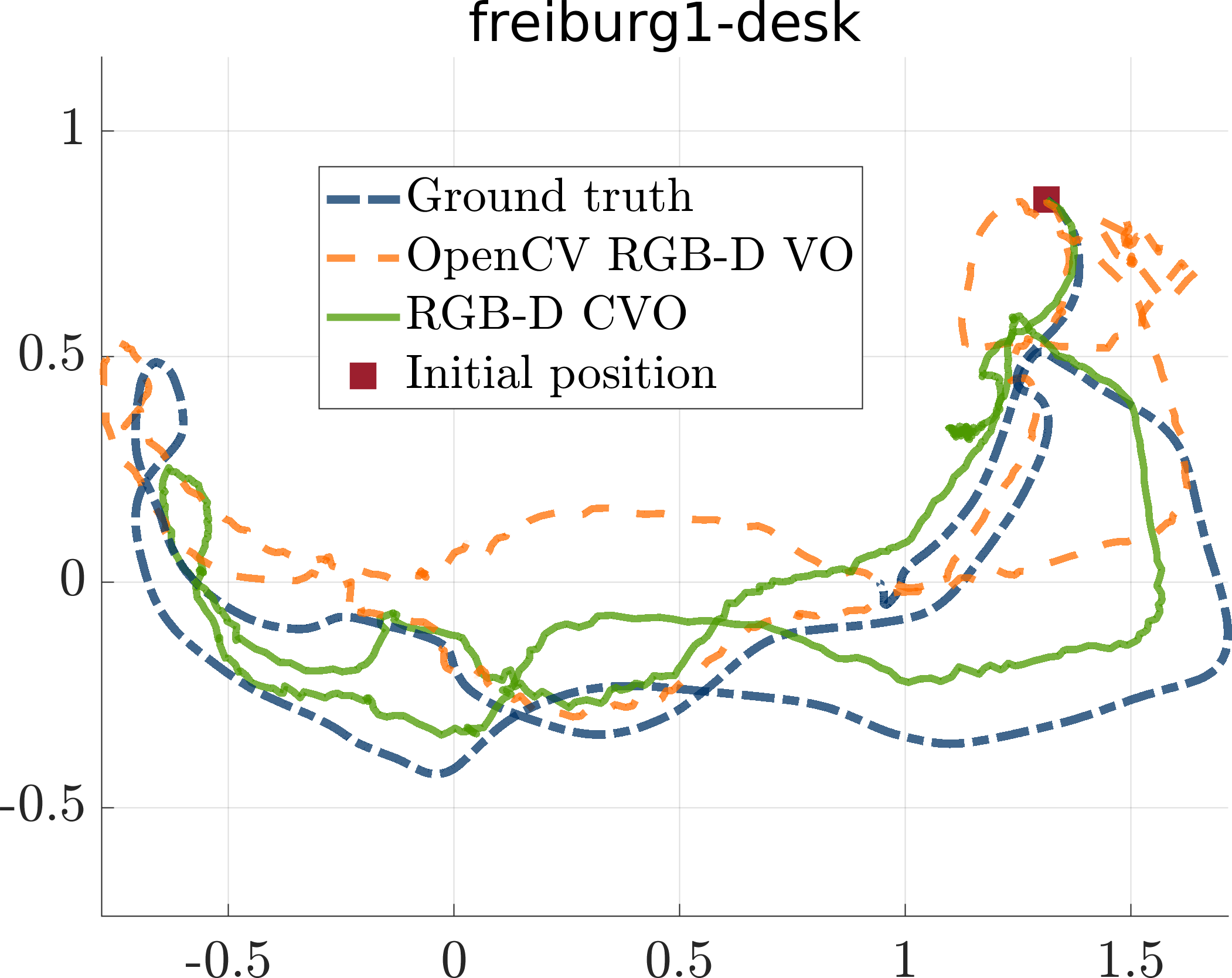}}
    \subfloat{\includegraphics[width=0.5\columnwidth]{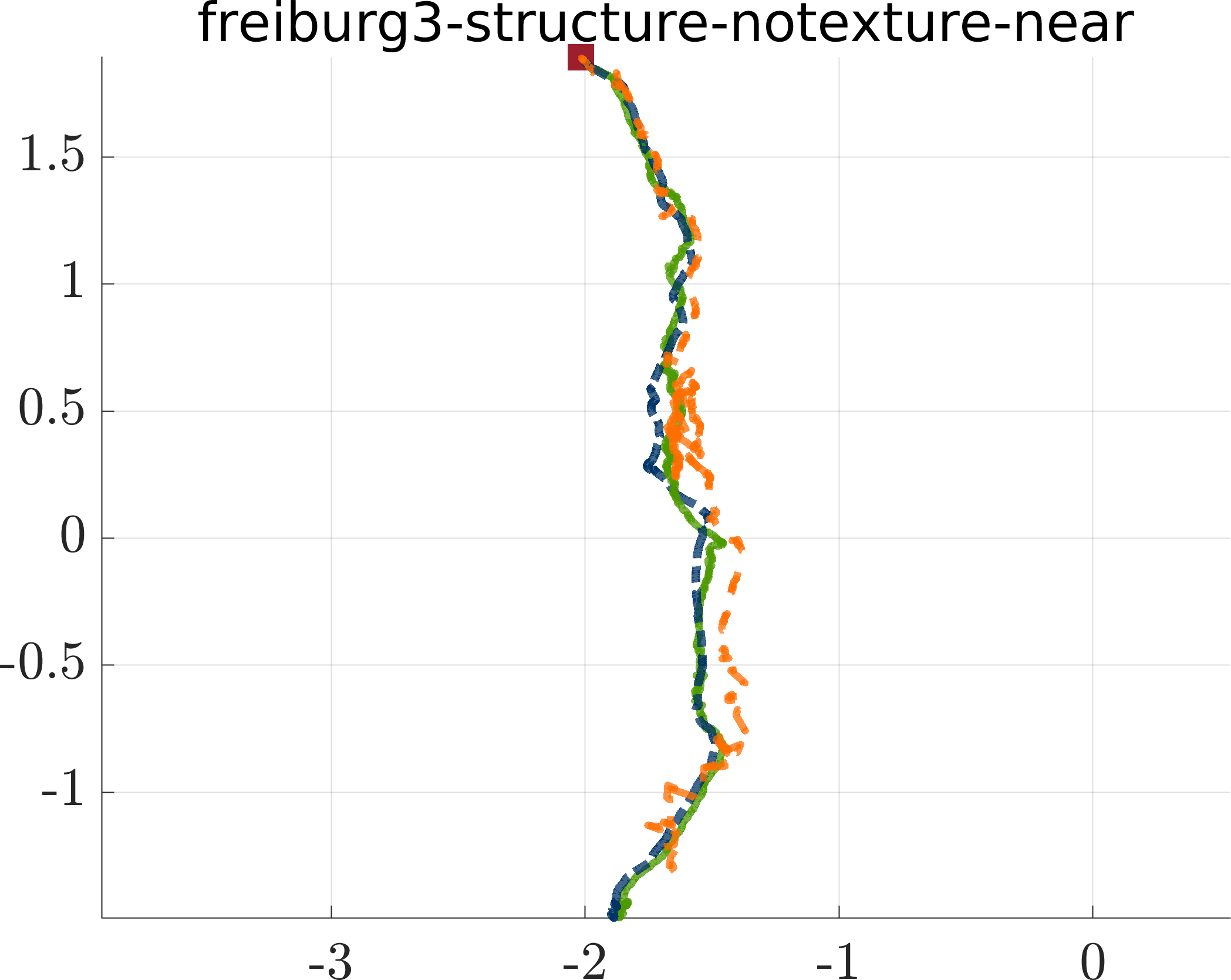}}
    \caption{Experimental evaluation of the proposed Continuous Visual Odometry (CVO) framework using the \mbox{RGB-D} SLAM dataset and benchmark~\citep{sturm12iros}. The baseline for comparison is the open source implementation of the~\citet{steinbrucker2011real} which is available through OpenCV~\citep{opencvrgbd}. The cumulative distribution functions of position and orientation errors shows performance of each method where the relative errors between two registered frames are computed using metrics defined earlier in~\ref{subsec:liesen}. The evaluation is done over the entire trajectories. For CVO, to reduce the number of points and improve the computational efficiency, a similar approach to~\citet{engel2018direct} is adopted by only maintaining the high gradient points. In this work, we extracted edges on 2D images and exploited the one-to-one correspondences in the Microsoft Kinect organized point cloud. The proposed method almost always performs better than or similar to the baseline; however, it is particularly interesting that the difference is significant when the lack of both structure and texture in the environment is evident (first column from left).}
    \label{fig:tum_rgbd}
    \squeezeup
\end{figure*}

\subsection{TUM RGBD Benchmark}

\begin{table}[t]
\begin{center}
\caption{ Parameters used for evaluation using TUM RGBD Benchmark, similar values were chosen for all experiments. The kernel characteristic length-scale is chosen to be adaptive as the algorithm converges; intuitively, we prefer a large neighborhood of correlation for each point, but as the algorithm reaches the convergence reducing the local correlation neighborhood allows for faster convergence and better refinement.}
\footnotesize
\begin{tabular}{l r}
\toprule
    Parameters & Value \\
        \midrule
        transformation convergence threshold $\epsilon$ & $1\mathrm{e}{-5}$ \\
        gradient norm convergence threshold $\epsilon$ & $5\mathrm{e}{-5}$ \\
        kernel characteristic length-scale $\ell$ & $0.15$ \\
        kernel characteristic length-scale $\ell$ (iteration $> 3$)  & $0.10$ \\
        kernel characteristic length-scale $\ell$ (iteration $> 10$)  & $0.06$ \\
        kernel characteristic length-scale $\ell$ (iteration $> 20$)  & $0.03$ \\
        kernel signal variance $\sigma$ & $0.1$  \\
        minimum step length & $0.2$ \\
        color space inner product scale & $10\mathrm{e}{-5}$ \\
        kernel sparsification threshold & $1\mathrm{e}{-3}$ \\
\bottomrule
\end{tabular}
\label{tab:parameters}
\end{center}
\squeezeup\squeezeup
\end{table}

We performed experiments on four RGB-D image sequences of \mbox{RGB-D} SLAM dataset and benchmark~\citep{sturm12iros}. This dataset was collected indoors with a Microsoft Kinect using a motion capture system as a proxy for ground truth trajectory. For all tracking experiments, the entire images were used sequentially without any skipping, i.e., at full framerate. Also, the same set of parameters was used for all experiments which are listed in Table~\ref{tab:parameters}. To improve the computational efficiency, we adopted a similar approach to~\citet{engel2014lsd} by maintaining a semi-dense point cloud for each scan. In this work, we simply extracted edges on 2D images and exploited the one-to-one correspondences in the organized point cloud. 

Figure~\ref{fig:tum_rgbd} shows the cumulative distribution functions of position and orientation relative per frame errors. The figure also shows the top view of the accumulate trajectories in which the absolute accumulated error can be seen. The initial transformation always was set to the identity and the kernel evaluations were sparsified (an example is shown in Fig.~\ref{fig:first}) by setting any value less than $1\mathrm{e}{-3}$ to zero. The baseline for comparison is the open source implementation of the direct RGB-D visual odometry method in~\citet{steinbrucker2011real} which is available through OpenCV~\citep{opencvrgbd}. The proposed method performs better than or similar to the baseline, suggesting its potential for improving the visual front-end systems used for odometry or SLAM solutions. 

\begin{remark}
	Due to the fact that the images are in color, each point in the cloud has a color associated with it. This is encoded as a point in $\mathbb{R}^3$. Therefore, we will take $\Ical=\mathbb{R}^3$ and its inner product to be
	\begin{equation}
	\langle \cdot,\cdot\rangle_\Ical = \frac{1}{10^5}\langle \cdot,\cdot\rangle_3.
	\end{equation}
\end{remark}



\section{Conclusion and Future Work} 
\label{sec:conclusion}
We developed a continuous and direct formulation and solution for the RGB-D visual odometry problem. The proposed solution, given a calibrated camera, directly works on 3D data and possesses a sparse (or sparsified) structure as well as naturally parallelizable structures. Furthermore, due to the continuous representation, we neither require the association between two measurements sets nor the same number of measurements within each set. Theoretically speaking, given the promising results presented in this paper, the proposed method is an alternative to the core module of many modern visual odometry and tracking systems~\citep{newcombe2011dtam,forster2014svo,wang2017stereo,engel2018direct}. 

Building a visual tracking system such as the work of~\citet{engel2018direct} is an interesting future direction. Another interesting idea is guiding the gradient flow of $\SE(3)$ using an Inertial Measurement Unit (IMU)~\citep{forster2017manifold,eckenhoff2018closed,hartley2019contact} or combined contact-IMU process models~\citep{rhartley-2018b,hartley2019contact}. The propagation of the flow can be done between any camera frames as the visual data has a typically lower frequency.


A similar approach to the problem presented in this paper was studied in \citep{Bloch1992} and \citep{BROCKETT199179}. The authors show that it is possible to view the differential equation, $\dot{H} = [H,[H,N]]$, where $H$ and $N$ are elements of some Lie algebra, as a gradient flow on an adjoint orbit of the corresponding group \textit{so long as the group is compact}. This can then be used to solve linear programming problems where the data is encoded in $H(0)$ and $N$. Information about the solution to the linear programming problem can then be learned from the double bracket structure.
The function which generates the gradient is
\begin{equation}\label{eq:killing}
    F(\theta) = \kappa\left(Q,\mathrm{Ad}_{\theta}N\right),\quad \theta\in\Gcal,\quad Q,N\in\mathfrak{g},
\end{equation}
where $\kappa$ is the Killing form and $\mathrm{Ad}$ is the adjoint action. When the Lie group, $\Gcal$, is compact, the Killing form is negative definite, i.e. $-\kappa$ is a nondegenerate, $\mathrm{Ad}$-invaraint inner product on $\mathfrak{g}$. Equation \eqref{eq:killing} is closely related to Equation \eqref{eq:max}; however the group studied here is $\SE(n)$ which is noncompact (and furthermore, there do not exist any $\mathrm{Ad}$-invariant inner products on $\se(n)$). In future work we shall discuss how this framework can be used to solve Problem \ref{prob:problem} when the group is restricted to be $\SO(n)$.

\section*{Acknowledgments}
\small{Funding for M. Ghaffari is given by the Toyota Research Institute (TRI), partly under award number N021515, however this article solely reflects the opinions and conclusions of its authors and not TRI or any other Toyota entity. Funding for J. Grizzle was in part provided by TRI and in part by NSF Award No.~1808051. Funding for W. Clark and A. Bloch was provided by NSF DMS 1613819 and AFSOR FA9550-18-10028.}

\small 
\bibliographystyle{plainnat}
\bibliography{strings-full,ieee-full,references}

\end{document}